\RequirePackage{fix-cm}
\documentclass[smallextended,natbib]{svjour3}
\smartqed
\usepackage{amsmath,amssymb}
\usepackage{graphicx,booktabs,xcolor,array,tabularx}
\usepackage[expansion=false]{microtype}
\usepackage{enumitem}
\usepackage[hidelinks]{hyperref}
\setlist{leftmargin=1.4em,itemsep=2pt,topsep=3pt}
\newenvironment{proof}[1][\proofname]{\par\addvspace{\medskipamount}\noindent{\itshape #1.}\enspace\ignorespaces}{\qed\par\addvspace{\medskipamount}}
\journalname{Dynamic Games and Applications}
\makeatletter\def\makeheadbox{}\makeatother

\newcommand{\Vp}{V_p^{\max}}
\newcommand{\Ve}{V_e^{\max}}
\newcommand{\Om}{\Omega_p^{\max}}

\newcommand{\x}{\mathbf{x}}
\newcommand{\U}{\mathcal{U}}
\newcommand{\UP}{\ensuremath{\mathrm{UP}}}
\newcommand{\NUP}{\ensuremath{\mathrm{NUP}}}
\newcommand{\BUP}{\ensuremath{\mathrm{BUP}}}
\newcommand{\De}{\mathcal{D}_e}
\newcommand{\Dpt}{\mathcal{D}_{PT}}
\newcommand{\Bbar}{\mathcal{B}}
\newcommand{\Pset}{\mathcal{P}}
\newcommand{\Tset}{\mathcal{T}}
\newcommand{\Cset}{\mathcal{C}}
\newcommand{\Wt}{\widetilde W_T}
\newcommand{\sgn}{\operatorname{sgn}}
\newcommand{\PI}[1]{(Part~I, #1)}
\newcommand{\OR}[1]{Supplementary Material, Sect.~#1}


\begin{document}

\title{Minimum-Time Escape from the Conic Field of View of a Reeds--Shepp Pursuer: the Faster Evader}
\titlerunning{Escape from a Reeds--Shepp pursuer's field of view: the faster evader}
\author{Ubaldo Ruiz}
\authorrunning{U. Ruiz}
\institute{U. Ruiz \at
  Centro de Investigaci\'on Cient\'ifica y de Educaci\'on Superior de Ensenada (CICESE), Ensenada, Baja California, Mexico\\
  \email{uruiz@cicese.edu.mx}}
\date{}

\maketitle

\begin{abstract}
We continue the study of minimum-time escape of an omnidirectional evader from the forward conic field of view of a Reeds--Shepp pursuer, now for a faster evader, and solve the game for every speed ratio $\nu>1$ and every half-angle $0<\alpha<\pi/2$: the escape time is characterized by explicit characteristics and the selection of roots of a scalar equation, and it is the value of the game in nonanticipative strategies at every state. Above the end $E$ of the evader's dispersal line on the front axis, the evader chases the reversing pursuer along a universal line fed by tributaries. The characteristic from the boundary of the usable part is an involute of a circle of radius $\nu$ ending at a cusp, and two explicit thresholds decide whether it reaches the axis before the cusp and whether the jump of the escape time across it ends before the cusp or runs to it; in the latter case primaries fold back next to the barrier. The pursuer's optimal strategy is the relay of the slower case. The evader follows the tributaries and switches once to a primary, or commits in advance to a heading certified by the scalar equation of the primaries. Level-set and semi-Lagrangian computations are compared with the value.
\end{abstract}

\section{Introduction}\label{sec:intro}
A directional sensor on a moving vehicle observes only part of the plane, so escaping surveillance means leaving its field of view rather than outrunning the observer. In Part~I \citep{ruiz2026conic} we studied minimum-time escape of an omnidirectional evader from the unbounded forward cone of half-angle $0<\alpha<\pi/2$ of a Reeds--Shepp pursuer, a car that can drive forward or in reverse but cannot rotate in place \citep{reeds1990}. The evader minimizes the exit time and the pursuer maximizes it. Part~I proved the escape/holding classification for every ratio $\nu=\Ve/\Vp$ of the evader's and the pursuer's maximum speeds. Straight reversal holds the whole cone for $\nu\le\sin\alpha$; for $\sin\alpha<\nu\le1$ escape is confined to a dart bounded by a barrier; for $\nu>1$ escape is possible from every state. For the slower or equally fast evader, Part~I constructed the escape time with Isaacs' characteristics \citep{isaacs1965,lewin1994}, verified it in feedback form and identified it with the value of the game in nonanticipative strategies at every state except one corner of the holding region.

This paper treats the faster evader, $\nu>1$. The closest forward-only model is the conic surveillance--evasion game of \citet{lewin1979conic}; related surveillance games and differential-drive variants are surveyed in Part~I \citep{dobbie1966,lewin1975,lewin1989rocket,greenfeld1987,ruiz2024dubins,braun2025,weintraub2023,ruiz2023fov,saavedra2024}. A faster evader changes the problem in three ways. First, the evader's axis dispersal line of the slower case ends at a point $E$, where the viscosity test on the axis fails \PI{Proposition~4}. Second, above $E$ a singular arc appears: the evader runs straight at the pursuer, which backs away in a straight line, so the front axis becomes a universal line in Isaacs' sense, fed from both sides by tributaries. Third, the characteristic issued from the boundary of the usable part (the $\BUP$ characteristic) no longer closes a holding region. It is a semipermeable curve that ends at a cusp, and across part of it the escape time jumps by a finite amount.

\paragraph{Contributions} The main results are the following.
\begin{enumerate}[label=(\roman*)]
\item For $\nu>1$ we give the characteristic families in closed form (Section~\ref{sec:char}): the point $E$ (Proposition~\ref{prop:E}), the universal line and its tributaries, whose value is an explicit classical solution of the Isaacs equation on the whole upper half-plane (Proposition~\ref{prop:trib}), and the $\BUP$ characteristic, an involute of the circle of radius $\nu$ about the pursuer's turning centre that ends at a cusp on that circle (Proposition~\ref{prop:barfast}).
\item The condition $\nu\sin(\alpha+\cos\alpha/\nu)\le1$ defines a regime F1 in which the $\BUP$ characteristic reaches the front axis at or before its cusp. In F1 the primaries, the tributaries and their dispersal lines are explicit (Proposition~\ref{prop:F1}, Lemma~\ref{lem:cover}).
\item In F1, the synthesis is verified in feedback form (Theorem~\ref{thm:feedback}), and its escape time $W$ is the value of the game in nonanticipative strategies at \emph{every} state of the cone (Theorem~\ref{thm:value}). The pursuer uses the relay of Part~I. The evader follows the tributaries and switches once to a primary. Both guarantees can be realized by sampled state feedback, so the value is the same whether or not a player observes the opponent's current control.
\item In the complementary regime F2, primaries fold before reaching the axis. Whether the tributary value undercuts the barrier before the cusp is decided by comparing the tributary value at the cusp with the retro-time of the cusp, which splits F2 into F2a and F2b (Section~\ref{sec:F2}). In F2a the results of (iii) hold unchanged (Theorem~\ref{thm:F2a}); the evader's argument needs only margins at the axis and the barrier that hold before the cusp, not the sign condition on the primary headings used in F1. In F2b the synthesis has folded primaries next to the barrier and a jump with square-root decay at the cusp (Appendix~\ref{app:cusp}). It is the value of the game (Theorem~\ref{thm:F2b}): the evader commits to the heading of a primary, which guarantees its time against every pursuer control, and an auxiliary envelope gives the relay bound before the covering of the tributaries is known. The same holds on the curve that separates F2a from F2b, so the conic escape game is solved for every $\nu>1$ and every $\alpha\in(0,\pi/2)$.
\item Independent level-set computations agree with the value in F2b away from singular lines. In F1, level-set and semi-Lagrangian computations are compared with the proved value, which identifies their persistent offsets on the universal line as errors of the discretizations (Section~\ref{sec:num}).
\end{enumerate}
The F1 verification follows the strategy of Part~I, with three new ingredients: a covering lemma for the tributary sheet, an invariance argument at the end point of the barrier, where the escape time is continuous but not Lipschitz, and a separation lemma that lets the evader switch branches at a sampling instant.

Section~\ref{sec:setting} recalls the formulation and Section~\ref{sec:char} constructs the characteristics. Sections~\ref{sec:F1} and~\ref{sec:F2} solve regimes F1 and F2, Section~\ref{sec:num} gives the numerical comparisons and Section~\ref{sec:lewin} the comparison with the Dubins game.

\section{Formulation and results from Part I}\label{sec:setting}
We recall the setting of Sections~2--4 of Part~I in normalized units. With $\Om$ the pursuer's maximum turn rate, lengths are measured in units of the minimum turning radius $r=\Vp/\Om$ and times in units of $1/\Om$, so that $\Vp=\Om=r=1$ and $\Ve=\nu$; dimensional values follow from the scaling (7) of Part~I. In the reduced frame attached to the pursuer, with $x$ to its right and $y$ along its heading, the evader's position $\x=(x,y)^\top$ obeys
\begin{equation}\label{eq:reduced}
\dot x=u_x+\omega_py,\qquad \dot y=u_y-v_p-\omega_px,
\end{equation}
where $\|\mathbf u\|\le\nu$ is the evader's velocity and the pursuer's speed and turn rate lie in the bow tie $|\omega_p|\le|v_p|\le1$. The field of view is the closed cone $K=\{\x:|\varphi|\le\alpha\}\cup\{\mathbf 0\}$, where $\rho=\|\x\|$ and $\varphi=\operatorname{atan2}(x,y)$ is the bearing, and the payoff is the first strict exit time $t_f=\inf\{t\ge0:\x(t)\notin K\}$. The bow tie lets the pursuer change the sign of its speed instantaneously, an idealization of the Reeds--Shepp car. With $J=t_f$, the evader's and the pursuer's strategy values are
\[
V_E=\inf_{\beta}\sup_{\mathbf w}J(\beta[\mathbf w],\mathbf w),\qquad V_P=\sup_{\gamma}\inf_{\mathbf u}J(\mathbf u,\gamma[\mathbf u]),
\]
where $\mathbf w=(v_p,\omega_p)$ and $\mathbf u$ range over measurable controls, and $\beta$ (evader) and $\gamma$ (pursuer) are nonanticipative strategies without delay: maps between control sets such that controls that agree a.e.\ on $[0,t]$ have images that agree a.e.\ on $[0,t]$, for every $t$ \citep{elliott1972,bardi1997}. This is the setting of Section~2 of Part~I; $V_E\le V_P$ is not assumed a priori. A guarantee of one player in feedback form bounds the corresponding value from one side; the reverse bounds come from sampled counter-controls (Appendix~\ref{app:value}).

\paragraph{Usable part} The right and left rays $M_\pm$ of $K$ are usable ($\UP_\pm$) for $\rho<\rho_B=\nu-\sin\alpha$, nonusable ($\NUP$) for $\rho>\rho_B$, and $B_\pm=\rho_B(\pm\sin\alpha,\cos\alpha)^\top$ are the $\BUP$ points. On the right ray the pursuer's control $(v_p,\omega_p)=(-1,-1)$ gives inward angular velocity $\rho\dot\varphi\le\nu-\rho-\sin\alpha$ against every evader control \PI{Eq.~(8)}. For $\nu>1$ the apex is usable. Throughout, $W$ denotes the escape time constructed from the characteristics below, defined formally in~\eqref{eq:W} for regime F1 and in Section~\ref{sec:F2} for regime F2; $W=0$ on $\UP\cup\{\mathbf 0\}$.

\paragraph{Hamiltonian} With costate $\boldsymbol\lambda=(\lambda_x,\lambda_y)^\top$ and $\Lambda=\|\boldsymbol\lambda\|$,
\begin{equation}\label{eq:Hstar}
H^*(\x,\boldsymbol\lambda)=1-\nu\Lambda+|\lambda_y|+|\Phi_\omega|,\qquad \Phi_\omega=\lambda_xy-\lambda_yx,
\end{equation}
with saddle controls $\mathbf u^*=-\nu\boldsymbol\lambda/\Lambda$, $v_p^*=-\sgn\lambda_y$ and $\omega_p^*=\sgn\Phi_\omega$ \PI{Eq.~(11)--(12)}. Where a switching function vanishes the maximizer is not unique, and, as in Part~I, we select an admissible one: if $\lambda_y=0$ and $\Phi_\omega\ne0$, either sign of $v_p$ with $|v_p|=1$, so that the saturated turn rate lies in the bow tie; if $\Phi_\omega=0$, any $\omega_p$ with $|\omega_p|\le|v_p|$. A normal characteristic satisfies $H^*=0$; an abnormal one satisfies $H_A^*=H^*-1=0$. By Proposition~1 of Part~I, a normal characteristic with singular steering lies on the front axis, with $\boldsymbol\lambda=(0,\lambda_y)$, $\omega_p=0$, $|v_p|=1$ and $(\nu-1)|\lambda_y|=1$; such arcs exist only for $\nu>1$.

\paragraph{Primaries} The characteristics ending at $\rho_f(\sin\alpha,\cos\alpha)^\top\in\UP_+$ are, with $\theta=\alpha+\tau$ in retro-time $\tau$ and the saddle controls $(v_p,\omega_p)=(-1,-1)$, given by Eq.~(16) of Part~I, which in the normalized units $r=\Vp=\Om=1$, $\Ve=\nu$ reads
\begin{equation}\label{eq:primary}
\begin{aligned}
\psi&=\rho_f-\sin\theta+\sin\alpha,\qquad \eta=\cos\alpha-\cos\theta-\nu\tau,\\
x&=\psi\sin\theta+\eta\cos\theta,\qquad y=\psi\cos\theta-\eta\sin\theta,
\end{aligned}
\end{equation}
carrying $W=\tau$ and costate $\boldsymbol\lambda=(-\cos\theta,\sin\theta)^\top/g$, $g=\rho_B-\rho_f$ \PI{Eqs.~(14)--(15)}. Eliminating $\rho_f$, a point $\x$ lies on a primary at retro-time $\tau$ if and only if $G(\tau;\x)=0$, where
\begin{equation}\label{eq:Gscalar}
G(\tau;\x)=(x+1)\cos\theta-y\sin\theta-\cos\alpha+\nu\tau,
\end{equation}
and the launch parameter is then
\begin{equation}\label{eq:rhof}
\rho_f=(x+1)\sin\theta+y\cos\theta-\sin\alpha,\qquad G_\tau=\rho_B-\rho_f
\end{equation}
\PI{Eq.~(23)--(24)}. The $\BUP$ characteristic $\Bbar_+$ is the abnormal characteristic issued from $B_+$; its path is~\eqref{eq:primary} with $\rho_f=\rho_B$, i.e.\ $\psi=\nu-\sin\theta$ \PI{Section~6.2}. Throughout, $S=\sin\theta$, $C=\cos\theta$ and $q(\tau)=\cos\alpha-\nu\tau$.

\paragraph{Results of Part I used here} Besides the formulas above we use: the relay $v_p=-1$, $\omega_p=-\sgn x$, with Filippov sliding on $x=0$, which defines a unique solution and a nonanticipative pursuer strategy while $y>0$ \PI{Lemma~1}; the subsolution test on the axis \PI{Proposition~4}, which for $\nu>1$ holds at a point reached by a primary and its mirror image if and only if $\lambda_y\le1/(\nu-1)$; the evader's escape bound for $\nu>1$ \PI{Proposition~3(c)}; the feedback verification and its sampled versions \PI{Theorem~1, Lemmas~2--6}; and the pairing of sample-and-hold controls with nonanticipative strategies \PI{Appendix~B}. Where an argument of Part~I used $\nu\le1$ we say so and give the replacement.

\section{Characteristics for the faster evader}\label{sec:char}
Throughout this section $\nu>1$ and $0<\alpha<\pi/2$.

\subsection{Primaries}
Part~(i) of Proposition~2 of Part~I holds for every $\nu$: the map $(\rho_f,\tau)\mapsto\x$ of~\eqref{eq:primary} has Jacobian determinant $\rho_B-\rho_f>0$. Part~(ii), which says that every primary reaches the front axis before any switch, is where $\nu\le1$ enters. We record what remains for $\nu>1$.

\begin{lemma}[Primaries for $\nu>1$]\label{lem:primfast}
Let $\nu>1$ and $0<\rho_f<\rho_B$. Along the primary~\eqref{eq:primary}, the pursuer's steering does not switch before the first axis contact or before $\theta=\pi/2$, whichever comes first, and the speed does not switch before $\theta=\pi$. Moreover $x_\tau=-gC-qS$, where $g=\rho_B-\rho_f$, so $x_\tau<0$ while $q>0$ and $\theta<\pi/2$.
\end{lemma}
\begin{proof}
Let $b=\rho_f+\sin\alpha$, so $\psi=b-\sin\theta$. If $b\ge1$, then $\psi>0$ for $\theta<\pi/2$. If $b<1$, at the first zero $\theta_\omega=\arcsin b$ of $\psi$ we have $\eta<0$, because $\eta(\alpha)=0$ and $\eta'=\sin\theta-\nu<0$ for every $\nu>1$; hence $x=\eta\cos\theta_\omega<0$ there, while $x=\rho_f\sin\alpha>0$ at launch, so the axis is reached first. The speed switch is at $\theta=\pi$. For the last claim write $\x=\Bbar_+(\tau)-g(\sin\theta,\cos\theta)$, with $x_B=\nu S-1+qC$ the abscissa of $\Bbar_+$ \PI{proof of Proposition~6(c)}; then $x_\tau=-gC-qS$.
\end{proof}

For $\nu\le1$ every primary reaches the axis while $q>0$ \PI{Proposition~6}. For $\nu>1$ the sign of $q$ may change before the axis contact; this is the origin of the folded primaries of regime F2 (Section~\ref{sec:F2}).

\subsection{The end of the axis dispersal line}
The primaries that reach the front axis meet their mirror images there, and Proposition~4 of Part~I admits the contact only while $\lambda_y\le1/(\nu-1)$.

\begin{proposition}[The point $E$]\label{prop:E}
Let $S=\sin(\alpha+\tau)$, $C=\cos(\alpha+\tau)$ and
\begin{equation}\label{eq:fE}
f(\tau)=\cos\alpha-\nu\tau+\frac{\nu CS}{1+S}-C,\qquad 0\le\tau\le\tfrac{\pi}{2}-\alpha .
\end{equation}
Then $f$ has exactly one root $\tau_*$; let $S_*$, $C_*$ be the corresponding values. The primary with launch parameter $\rho_E=\nu-\sin\alpha-(\nu-1)S_*\ge1-\sin\alpha>0$ reaches the front axis for the first time at retro-time $\tau_*$, without leaving the cone and without a switch of the pursuer's controls, at
\begin{equation}\label{eq:E}
E=(0,y_E)^\top,\qquad y_E=\frac{\nu C_*}{1+S_*},\qquad W_E=\tau_*,
\end{equation}
where it arrives with $\lambda_y=1/(\nu-1)$. Axis contacts below $E$ satisfy the subsolution test, and those above $E$ violate it.
\end{proposition}
\begin{proof}
\emph{The root of $f$.} A primary that meets the axis at retro-time $\tau$ does so at the height $y_D(\tau)=(C-\cos\alpha+\nu\tau)/S$, by~\eqref{eq:Gscalar} with $x=0$, with gradient $(-C,S)/g$, so $\lambda_y=S/g$, where $g=\rho_B-\rho_f=\nu-S-y_DC$ by~\eqref{eq:rhof}. Differentiating, $y_D'=g/S$. Since $\nu C^2/(1+S)=\nu(1-S)$,
\[
g-(\nu-1)S=\nu(1-S)-\frac CS\,(C-\cos\alpha+\nu\tau)=\frac CS\,f(\tau),
\]
so the test $\lambda_y\le1/(\nu-1)$, i.e.\ $g\ge(\nu-1)S$, is equivalent to $f(\tau)\ge0$. Now $f(0)=\nu\cos\alpha\sin\alpha/(1+\sin\alpha)>0$; $f(\pi/2-\alpha)=\cos\alpha-\nu(\pi/2-\alpha)<0$, because $\pi/2-\alpha>\cos\alpha$ and $\nu>1$; and $f'(\tau)=S[1-\nu(S+2)/(S+1)]<0$. Hence $f$ has exactly one root $\tau_*$, with $f>0$ before it and $f<0$ after it, and $y_D$ increases with $\tau$ on the axis contacts of admissible primaries, where $g>0$.

\emph{The primary through $E$.} Let $g_*=(\nu-1)S_*$. The primary with $\rho_E=\rho_B-g_*$ has $g=g_*$ along it. Differentiating~\eqref{eq:primary}, $x'=-qS-g_*C$. Thus $x'/S=-q-g_*\cot(\alpha+\tau)$ is increasing, because $q$ and the cotangent decrease, and $x'(\tau_*)=-\nu S_*C_*/(1+S_*)<0$ by $f(\tau_*)=0$. So $x'<0$ on $[0,\tau_*]$, and $x$ decreases from $\rho_E\sin\alpha>0$ to its first zero at $\tau_*$. The distance $d=y\sin\alpha-x\cos\alpha$ from $M_+$ satisfies $d'=g_*\cos\tau+q\sin\tau$ \PI{proof of Proposition~6(d)}. Since $d'/\sin\tau=g_*\cot\tau+q$ is decreasing, $d'$ changes sign at most once, from positive to negative. As $d(0)=0$ and $d(\tau_*)=y_E\sin\alpha>0$, $d>0$ on $(0,\tau_*]$, and the primary stays in the cone. Finally $0<\rho_E<\rho_B$, and Lemma~\ref{lem:primfast} excludes a switch.
\end{proof}

The evader's axis dispersal line is $\De=\{0\}\times(0,y_E)$. Numerically, $(y_E,W_E,\rho_E)=(0.55154,0.46673,0.86653)$ for $(\nu,\alpha)=(1.1,10^\circ)$ and $\rho_E=0.58387$ for $(1.5,30^\circ)$.

\subsection{The universal line and its tributaries}\label{sec:univ}
By Proposition~1(ii) of Part~I, the only singular arc is the front axis with
\begin{equation}\label{eq:lamU}
\boldsymbol\lambda_U=(0,\Lambda_U)^\top,\qquad \Lambda_U=\frac{1}{\nu-1},
\end{equation}
evader control $\mathbf u^*=(0,-\nu)$ and pursuer control $v_p=-1$, $\omega_p=0$, a straight backward motion, which an actual Reeds--Shepp pursuer realizes exactly. Above $E$ the evader closes at the relative speed $\nu-1$, and
\begin{equation}\label{eq:U}
\U=\{(0,y)^\top:\ y>y_E\},\qquad W_U(y)=W_E+\frac{y-y_E}{\nu-1}.
\end{equation}
At $E$ the tangential derivative of $W$ along the axis is continuous, because the arriving primary has $\lambda_y=\Lambda_U$. $W_U$ is consistent with the bound $\rho/(\nu-1)$ of Proposition~3(c) of Part~I, since $W_E\le y_E/(\nu-1)$: indeed $y_ES_*=C_*-\cos\alpha+\nu W_E\ge(\nu-1)W_E$ because $\cos(\alpha+\tau)\ge\cos\alpha-\tau$, and $S_*\le1$.

\paragraph{Tributaries} Characteristics leave $\U$ on both sides with saturated steering; on the right, $(v_p,\omega_p)=(-1,-1)$. With retro-time $\sigma$ measured from the departure point $(0,s)$, $s>y_E$,
\begin{equation}\label{eq:trib}
x=(s+\nu\sigma)\sin\sigma-(1-\cos\sigma),\qquad y=(s+\nu\sigma)\cos\sigma-\sin\sigma,
\end{equation}
with $\boldsymbol\lambda=\Lambda_U(\sin\sigma,\cos\sigma)^\top$ and $W=W_U(s)+\sigma$. The family is inverted in closed form: for $y>0$, let
\begin{equation}\label{eq:tribinv}
R=\sqrt{y^2+|x|(|x|+2)},\qquad \sigma=2\arctan\frac{|x|}{R+y},\qquad s=R-\nu\sigma .
\end{equation}
Then $(|x|+1)\cos\sigma-y\sin\sigma=1$ and $(|x|+1)\sin\sigma+y\cos\sigma=R$, so $(s,\sigma)$ are the parameters of the tributary through $\x$ (of its mirror image for $x<0$), and the value it carries is
\begin{equation}\label{eq:WT}
\Wt(\x)=W_E+\frac{R-y_E-\sigma}{\nu-1}.
\end{equation}
We write $W_T$ for $\Wt$ on the tributary sheet $\{s\ge y_E\}$, where it is the value carried by a genuine tributary; on the axis $\Wt=W_U$.

\begin{proposition}[Tributary field]\label{prop:trib}
$\Wt$ is a classical solution of the Isaacs equation $H^*(\x,\nabla\Wt)=0$ on the open half-plane $\{y>0\}$. It is $C^1$ there, its gradient is Lipschitz on compact subsets, and for $x\ge0$
\begin{equation}\label{eq:gradT}
\nabla\Wt=\frac{(\sin\sigma,\cos\sigma)^\top}{\nu-1},\qquad \Phi_\omega=-\frac{1-\cos\sigma}{\nu-1}\le0,
\end{equation}
so the saddle controls are $\mathbf u_T=-\nu(\sin\sigma\,\sgn x,\cos\sigma)^\top$ and the relay $(v_p,\omega_p)=(-1,-\sgn x)$. Moreover:
\begin{enumerate}[label=(\alph*)]
\item under the relay, with Filippov sliding on $x=0$, $\frac{d}{dt}\Wt(\x(t))\ge-1$ for every measurable evader control;
\item under $\mathbf u_T$ and any admissible pursuer control,
\begin{equation}\label{eq:WTdot}
\frac{d}{dt}\Wt(\x(t))=-1-\frac{\Delta_v\cos\sigma+\Delta_\omega(1-\cos\sigma)}{\nu-1}\le-1,
\end{equation}
where $\Delta_v=v_p+1\ge0$ and $\Delta_\omega=\omega_p+1\ge0$ for $x\ge0$, with the mirrored quantities for $x<0$.
\end{enumerate}
\end{proposition}
\begin{proof}
In the coordinates $(R,\sigma)$, $x+1=\cos\sigma+R\sin\sigma$ and $y=R\cos\sigma-\sin\sigma$ for $x\ge0$, with Jacobian determinant $-R\ne0$. Since $(\nu-1)\Wt=R-\sigma+\text{const}$, this gives~\eqref{eq:gradT}; the sign of $\Phi_\omega$ uses $y\sin\sigma-x\cos\sigma=\cos\sigma-1$. Then $\lambda_y>0$ and $H^*=1-\nu/(\nu-1)+\cos\sigma/(\nu-1)+(1-\cos\sigma)/(\nu-1)=0$. As $x\to0^\pm$, $\sigma\to0$ and $\nabla\Wt\to(0,1)^\top/(\nu-1)$ from both sides. Since $\sigma=O(|x|)$ and $R$ is Lipschitz, $\sin\sigma\,\sgn x$ is Lipschitz, and so are $\nabla\Wt$ and $\mathbf u_T$. Under $(-1,-1)$, $\dot\x=(u_x-y,\,u_y+1+x)$ and
\[
(\nu-1)\tfrac{d}{dt}\Wt=\mathbf u^\top(\sin\sigma,\cos\sigma)^\top+(x+1)\cos\sigma-y\sin\sigma\ge-\nu+1 .
\]
During sliding on the axis, $\dot x=0$ and $\dot y=u_y+1\ge1-\nu$, so $\frac{d}{dt}\Wt=\dot y/(\nu-1)\ge-1$. This proves (a). For (b), the deviation from $(-1,-1)$ adds $(\Delta_\omega y,\,-\Delta_v-\Delta_\omega x)$ to $\dot\x$, and the same identities give~\eqref{eq:WTdot}. On the bow tie $|\omega_p|\le|v_p|\le1$, both $\Delta_v$ and $\Delta_\omega$ are nonnegative.
\end{proof}

The identities in this proof were also verified symbolically (\OR{S7}). The tributaries in which the pursuer turns toward the evader therefore feed the universal line $\U$ from both sides, with the same relay as the primaries; $\U$ is a universal line of Isaacs' type \citep{isaacs1965,lewin1994}. In the optimal forward play, an evader far from the apex first moves toward the axis, then runs straight at the backing pursuer down to $E$, and finally veers to one side and escapes along a primary (Fig.~\ref{fig:fwd}).

\subsection{The barrier and its cusp}\label{sec:barfast}
Near $B_+$ the usable part is thin: the escape speed $\nu-\sin\alpha-\rho_f=\rho_B-\rho_f$ at the terminal point, the inverse of the terminal costate magnitude \PI{Eq.~(15)}, tends to zero as $\rho_f\uparrow\rho_B$. States on one side of the $\BUP$ characteristic escape quickly through this part of the $\UP$, whereas states on the other side cannot be driven across it and must escape elsewhere. Let
\begin{equation}\label{eq:taucusp}
\tau_c=\frac{\cos\alpha}{\nu},\qquad \theta_c=\alpha+\tau_c=\alpha+\frac{\cos\alpha}{\nu}.
\end{equation}

\begin{proposition}[$\BUP$ characteristic for $\nu>1$]\label{prop:barfast}
Let $\Bbar_+$ be the $\BUP$ characteristic issued from $B_+$, with costate direction $\mathbf n(\theta)=(-\cos\theta,\sin\theta)^\top$.
\begin{enumerate}[label=(\roman*)]
\item For $\theta\in[\alpha,\pi)$ the pursuer's controls along $\Bbar_+$ do not switch, $(v_p,\omega_p)=(-1,-1)$, and $H_A^*\equiv0$. This holds in particular on $[0,\tau_c]$, because $\theta_c<\pi/2$.
\item Along $\Bbar_+$,
\begin{equation}\label{eq:bartangent}
(x',y')=-q\,(\sin\theta,\cos\theta),\qquad d'=q\sin(\theta-\alpha),\qquad \rho^2\varphi'=q\eta,
\end{equation}
where $d=y\sin\alpha-x\cos\alpha$ and $\eta<0$ for $\tau>0$. Hence, on $(0,\tau_c)$, $x$ and $y$ decrease, the distance from $M_+$ increases and the bearing decreases. At $\tau_c$ the velocity vanishes and reverses, and $\Bbar_+$ has a cusp at
\begin{equation}\label{eq:Pc}
P_c=\bigl(\nu\sin\theta_c-1,\ \nu\cos\theta_c\bigr)^\top .
\end{equation}
\item $\Bbar_+$ reaches the front axis at or before its cusp if and only if $\nu\sin\theta_c\le1$. Otherwise it stays in the open right half of $K$, at a positive distance from $M_+$, on $(0,\tau_c]$.
\item $\Bbar_+$ is semipermeable: by playing $(-1,-1)$ the pursuer ensures $\mathbf n^\top\dot\x\ge0$ whatever the evader does, and by playing $\mathbf u=-\nu\mathbf n$ the evader ensures $\mathbf n^\top\dot\x\le0$ whatever the pursuer does.
\end{enumerate}
\end{proposition}
\begin{proof}
(i) Since $\nu>1>\sin\theta$, $\psi=\nu-\sin\theta>0$, so $\Phi_\omega<0$ and $\omega_p^*=-1$; since $\lambda_y\propto\sin\theta>0$ on $(0,\pi)$, $v_p^*=-1$. Then $H_A^*/\Lambda=-\nu+\sin\theta+\psi=0$. The function $\alpha\mapsto\alpha+\cos\alpha$ increases on $(0,\pi/2)$ to $\pi/2$, so $\theta_c<\alpha+\cos\alpha<\pi/2$. (ii) Differentiating~\eqref{eq:primary} with $\psi'=-\cos\theta$ and $\eta'=\sin\theta-\nu$ gives $x'=\cos\theta(\psi-\nu)-\eta\sin\theta=-\sin\theta(\cos\theta+\eta)=-q\sin\theta$, and $y'$ is analogous. The expressions for $d'$ and $\rho^2\varphi'=x'y-y'x$ follow from $y\sin\theta-x\cos\theta=-\eta$. Since $\eta(0)=0$ and $\eta'<0$, $\eta<0$ for $\tau>0$. The sign of $q$ changes at $\tau_c$, which gives the cusp and the minimum of the bearing; at $\tau_c$, $\eta=-\cos\theta_c$, which gives~\eqref{eq:Pc}. (iii) By (ii), $x$ decreases strictly on $[0,\tau_c]$ from $\rho_B\sin\alpha>0$ to $x(P_c)=\nu\sin\theta_c-1$, and $d$ increases from $0$. (iv) By~\eqref{eq:reduced}, $\mathbf n^\top\dot\x=\mathbf n^\top\mathbf u-v_p\sin\theta+\omega_p\Phi_\omega/\Lambda$. With $(v_p,\omega_p)=(-1,-1)$ this is at least $-\nu+\sin\theta+\psi=0$; with $\mathbf u=-\nu\mathbf n$ it is at most $-\nu+\sin\theta+\psi=0$.
\end{proof}

\paragraph{Geometry of the $\BUP$ characteristic} By~\eqref{eq:primary} with $\psi=\nu-\sin\theta$,
\begin{equation}\label{eq:involute}
(x+1,\,y)=\nu(\sin\theta,\cos\theta)+q\,(\cos\theta,-\sin\theta),\qquad q=\nu(\theta_c-\theta),
\end{equation}
so $\Bbar_+$ is an involute of the circle of radius $\nu$ centred at $(-1,0)$: the point lies on the tangent to that circle at $\nu(\sin\theta,\cos\theta)-(1,0)$, at the distance $q$ equal to the arc still to be unwound. The cusp $P_c$ is the point where the involute reaches the circle, $q=0$. On that circle the evader's speed $\nu$ equals the speed with which the turning pursuer sweeps the plane, so the barrier cannot be continued smoothly beyond it. This is the mechanism of the barrier cusp in the forward-only game, whose barrier ``winds the $1/w$-circle'' and ends where it reaches it \citep[Section~6]{lewin1979conic}; see Section~\ref{sec:lewin}.

We write $P_{c\pm}$ for the cusp and its mirror image. By Proposition~\ref{prop:barfast}(iii), the condition
\begin{equation}\label{eq:alphaF}
\nu\sin\Bigl(\alpha+\frac{\cos\alpha}{\nu}\Bigr)>1
\end{equation}
defines \emph{regime F2}, in which $\Bbar_+$ ends at its cusp inside the right half of the cone; \emph{regime F1} is its complement within $\nu>1$. For each $\alpha$ the left side of~\eqref{eq:alphaF} increases with $\nu$, so the boundary $\alpha=\alpha_F(\nu)$ is a single curve (Fig.~\ref{fig:regimes}). F1 consists of narrow cones for moderate $\nu$: $\alpha_F=21.5^\circ$ at $\nu=1.05$, $9.3^\circ$ at $\nu=1.2$ and $1.4^\circ$ at $\nu=2$, widening toward $\alpha=\pi/2$ as $\nu\downarrow1$ (e.g.\ $53.0^\circ$ at $\nu=1.001$). This classification concerns the geometric characteristic only; which branch carries the value near the cusp is decided by a second comparison, which splits F2 into F2a and F2b along a curve $\alpha=\alpha_Z(\nu)$ (Section~\ref{sec:F2}).

\begin{figure}[!htbp]
\centering
\includegraphics[width=0.8\textwidth]{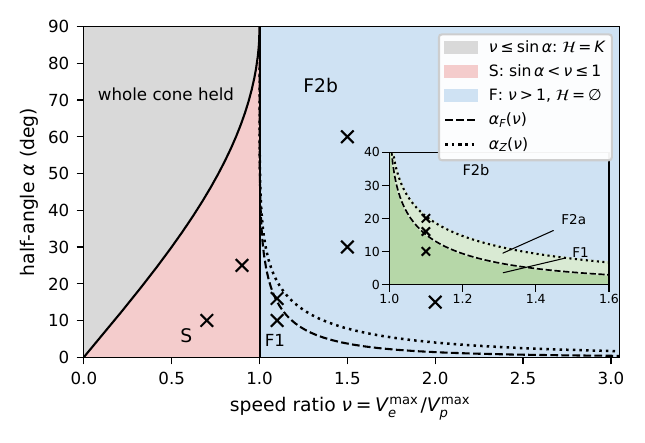}
\caption{Regimes of the conic escape game. Solid curves: $\nu=\sin\alpha$ and $\nu=1$ (Part~I). Dashed: $\alpha_F(\nu)$, at or below which the $\BUP$ characteristic reaches the front axis at or before its cusp (regime F1, Section~\ref{sec:F1}); above it lies regime F2 (Section~\ref{sec:F2}). Dotted: $\alpha_Z(\nu)$, which separates F2a, where the jump ends before the cusp, from F2b. The value is proved in F1, F2a and F2b (Sections~\ref{sec:F1}, \ref{sec:F2a} and~\ref{sec:F2b}), and on the curve $\alpha_Z$. The inset enlarges $1<\nu\le1.6$. Crosses: examples used here and in Part~I.}
\label{fig:regimes}
\end{figure}

\section{Regime F1: solution and value of the game}\label{sec:F1}
Throughout this section $\nu>1$ and regime F1 holds:
\begin{equation}\label{eq:F1}
\nu\sin\theta_c\le1 .
\end{equation}
By Proposition~\ref{prop:barfast}(iii), $\Bbar_+$ reaches the front axis at or before its cusp, at the first root $\theta_I$ of $\Gamma(\theta)=\nu\sin\theta-1+\cos\alpha\cos\theta-\nu(\theta-\alpha)\cos\theta$, at the point $I=(0,y_I)^\top$ and retro-time $\tau_I=\theta_I-\alpha$, with
\begin{equation}\label{eq:yI}
y_I=\frac{\nu-\sin\theta_I}{\cos\theta_I},\qquad \alpha<\theta_I\le\theta_c<\frac{\pi}{2},
\end{equation}
by the computation of Proposition~5 of Part~I, which uses only the path of $\Bbar_+$. Indeed $\Gamma'(\theta)=\sin\theta\,[\nu(\theta-\alpha)-\cos\alpha]<0$ on $(\alpha,\theta_c)$ and $\Gamma(\theta_c)=\nu\sin\theta_c-1\le0$, so $\theta_I<\theta_c$ in the interior of F1 and $\theta_I=\theta_c$, i.e.\ $I=P_c$, on its boundary $\nu\sin\theta_c=1$. In both cases $q>0$ on $[0,\tau_I)$, which is all that the arguments below use. Unlike the slower case, this contact does not close a holding region: the universal line and its tributaries supply values above $I$. Theorem~\ref{thm:F2a} in Section~\ref{sec:F2a} extends the results of this section to the part F2a of regime F2.

\subsection{The junction of primaries and tributaries}
For a label $\rho_f\in(0,\rho_B]$, with $\rho_f=\rho_B$ giving $\Bbar_+$, let $F(\rho_f,\tau)=\Wt(\x(\rho_f,\tau))-\tau$, so that a primary point carries the tributary value exactly where $F=0$.

\begin{proposition}[Junction in regime F1]\label{prop:F1}
Assume~\eqref{eq:F1}.
\begin{enumerate}[label=(\roman*)]
\item $y_I>y_E$ and $W_U(y_I)<\tau_I$.
\item Along a primary, $F_\tau=-\frac{\nu}{\nu-1}\bigl(1-\sin(\theta-\sigma)\bigr)$, where $\sigma$ is evaluated at $\x(\tau)$. On every arc used below, $F_\tau<0$.
\item Along $\Bbar_+$, $F$ has a unique root $\tau_Z\in(0,\tau_I)$. We write $Z_+=\Bbar_+(\tau_Z)$ and $Z_-$ for its mirror image.
\item For $\rho_E<\rho_f<\rho_B$, the primary meets the tributary sheet at a unique point where $F=0$, at a retro-time $\tau_J(\rho_f)$ before its axis contact. These points form two smooth curves $\Dpt^\pm$, dispersal lines of the evader, which join $E$ to $Z_\pm$.
\item The tributary sheet is admissible, $s>y_E$, along $\Dpt^\pm$ and along the arcs of the $\BUP$ characteristics from $Z_\pm$ to $I$.
\end{enumerate}
\end{proposition}
\begin{proof}
(i) By the proof of Proposition~\ref{prop:E}, the axis contacts of the primaries with labels in $[\rho_E,\rho_B]$ have heights $y_D(\tau)$ that increase with $\tau\in[\tau_*,\tau_I]$, from $y_E$ to $y_I$, and $d\tau/dy_D=\lambda_y>1/(\nu-1)$ for $\tau>\tau_*$ because $f(\tau)<0$ there. Integrating from $E$ to $I$ gives $\tau_I-W_E>(y_I-y_E)/(\nu-1)=W_U(y_I)-W_E$, and $y_I>y_E$.

(ii) Along a primary, $(x',y')=(-\nu\cos\theta+y,\ \nu\sin\theta-1-x)$. With~\eqref{eq:gradT} and $(x+1)\cos\sigma-y\sin\sigma=1$,
\[
\nabla\Wt^\top(x',y')^\top=\frac{\nu\sin(\theta-\sigma)-1}{\nu-1},
\]
and subtracting $1$ gives the formula. Equality $F_\tau=0$ requires $\theta-\sigma=\pi/2$. On the arcs used below, $0\le\sigma<\pi/2$ and $\alpha\le\theta\le\theta_I<\pi/2$, because these arcs end at or before their axis contact, which occurs by $\tau_I$; hence $F_\tau<0$.

\emph{Positivity at launch.} For the primary launched at $\rho_E$, $F(\rho_E,\tau_*)=0$ at $E$ and $F_\tau<0$ on $[0,\tau_*]$, so $F(\rho_E,0)>0$. Along the right ray, $\frac{d}{d\rho}\Wt(\rho\sin\alpha,\rho\cos\alpha)=\cos(\alpha-\sigma)/(\nu-1)>0$, so $F(\rho_f,0)>0$ for every $\rho_f\ge\rho_E$, including $B_+$.

(iii) Along $\Bbar_+$, $F(\rho_B,0)>0$ and $F(\rho_B,\tau_I)=W_U(y_I)-\tau_I<0$ by (i), because $\Wt=W_U$ on the axis. Since $F(\rho_B,\cdot)$ is strictly decreasing, it has a unique root $\tau_Z\in(0,\tau_I)$.

(iv) For $\rho_E<\rho_f<\rho_B$, $F(\rho_f,0)>0$, and at the axis contact $\tau_D(\rho_f)$, which lies above $E$, $F<0$ by the integration in (i). Since $F(\rho_f,\cdot)$ is strictly decreasing, it has a unique root $\tau_J(\rho_f)\in(0,\tau_D(\rho_f))$. Because $F_\tau<0$, the implicit function theorem makes $\tau_J$ smooth, so the points $\x(\rho_f,\tau_J(\rho_f))$ form a smooth curve. It tends to $E$ as $\rho_f\downarrow\rho_E$ and to $Z_+$ as $\rho_f\uparrow\rho_B$.

(v) Differentiating~\eqref{eq:tribinv} gives $ds=\sin\sigma\,dx+\cos\sigma\,dy-\frac{\nu-1}{R}(\cos\sigma\,dx-\sin\sigma\,dy)$. Let $\delta=\theta-\sigma$. With $\partial\x/\partial\rho_f=(\sin\theta,\cos\theta)$ and $\partial\x/\partial\tau=(x',y')$,
\[
s_{\rho_f}=\cos\delta-\frac{\nu-1}{R}\sin\delta,\qquad s_\tau=-\nu(1-\sin\delta)+\frac{(\nu-1)\nu}{R}\cos\delta .
\]
Along $\Dpt^+$, $F_{\rho_f}=\cos\delta/(\nu-1)$ and $F_\tau=-\nu(1-\sin\delta)/(\nu-1)$ give $d\tau_J/d\rho_f=\cos\delta/(\nu(1-\sin\delta))$, and
\[
\frac{ds}{d\rho_f}=s_{\rho_f}+s_\tau\frac{d\tau_J}{d\rho_f}=\frac{\nu-1}{R}\Bigl(\frac{\cos^2\delta}{1-\sin\delta}-\sin\delta\Bigr)=\frac{\nu-1}{R}>0 .
\]
Since $s=y_E$ at $E$, $s>y_E$ along $\Dpt^\pm$. Along $\Bbar_+$, while $x\ge0$, the involute form~\eqref{eq:involute} and the identities after~\eqref{eq:tribinv} give $\nu\sin\delta+q\cos\delta=1$ and $\nu\cos\delta-q\sin\delta=R$. Substituting $1-\sin\delta=(\nu-1)\sin\delta+q\cos\delta$ into $s_\tau$ gives
\begin{equation}\label{eq:sbar}
s_\tau=-\frac{\nu}{R}\,q\,(1-\sin\delta)\qquad\text{along }\Bbar_+ .
\end{equation}
Since $q>0$ on $[0,\tau_I)$ and $|\delta|<\pi/2$, $s$ decreases strictly from $B_+$ to $I$, and $s\ge s(I)=y_I>y_E$ on the whole arc.
\end{proof}

For $(\nu,\alpha)=(1.1,10^\circ)$, $y_I=0.55756$, $\tau_I=0.60655$, $W_U(y_I)=0.52691$, $Z_\pm=(\pm0.00541,0.56314)$ and $\tau_Z=0.58303$. Near $E$, $\Dpt^\pm$ have half-width $|x|\simeq40(y-y_E)^2$, so the channel between them is only about $8\times10^{-5}$ wide at $0.001$ above $E$ (Fig.~\ref{fig:F1}b). Such a narrow channel is hard to resolve on a grid (Section~\ref{sec:num}).

\subsection{The primary chart and the tributary sheet}
Let $\Cset_+$ be the open \emph{primary chart} of the right half, the image of $\{(\rho_f,\tau):0<\rho_f<\rho_B,\ 0<\tau<\tau_D(\rho_f)\}$ under~\eqref{eq:primary}, where $\tau_D$ is the first axis contact.

\begin{lemma}[Primary chart in F1]\label{lem:prim}
Assume~\eqref{eq:F1}. Every primary with $0<\rho_f<\rho_B$ reaches the axis at a time $\tau_D(\rho_f)<\tau_I$, with $x_\tau<0$ and $\psi>0$ before it, so primaries in $\Cset_+$ do not fold. On $\Cset_+$, the primary value $W_P$, the first positive root of~\eqref{eq:Gscalar}, is the unique root in $(0,\tau_I)$ with $G_\tau>0$; it is smooth, with $\nabla W_P=(-C,S)^\top/g$. $\Cset_+$ is the region bounded by $\UP_+$, the axis segment $\{0\}\times(0,y_I)$ and the arc of $\Bbar_+$ from $B_+$ to $I$. Finally $\nu S<1$ on $\Cset_+$.
\end{lemma}
\begin{proof}
In \PI{Proposition~6} the hypothesis $\nu\le1$ enters only through $q>0$ on the retro-times considered: on $(0,\tau_I)$ in step (b), where every root lies in an interval on which $G$ is concave; on $[0,\tau_I)$ in step (c), where $x_\tau=-gC-qS$, which is still negative at $\tau_I$ because $g>0$; and on $(0,\tau_I)$ in step (d), where the distance of $\Bbar_+$ from $M_+$ increases. Under~\eqref{eq:F1}, $q>0$ on $[0,\tau_I)$, so steps (a)--(d) apply verbatim and give all statements except the last two. Uniqueness is asserted only on $(0,\tau_I)$: the analytic continuation of~\eqref{eq:Gscalar} may have further roots with $G_\tau>0$ beyond $\tau_I$ (e.g.\ at $\tau=1.925$ for a point of the chart of $(1.1,10^\circ)$ with $\tau_I=0.607$), which belong to no admissible arc. For $\psi>0$ and the absence of folding use Lemma~\ref{lem:primfast}. On $\Cset_+$, $W_P<\tau_D(\rho_f)<\tau_I$, so $\theta<\theta_I\le\theta_c$ and $\nu S<\nu\sin\theta_c\le1$.
\end{proof}

\begin{lemma}[Primary and tributary regions]\label{lem:region}
Assume~\eqref{eq:F1}. Let $\tau_P=\tau_D$ on $(0,\rho_E]$ and $\tau_P=\tau_J$ on $(\rho_E,\rho_B)$, let $\Sigma_P=\{(\rho_f,\tau):0<\rho_f<\rho_B,\ 0<\tau<\tau_P(\rho_f)\}$, and let $\Pi$ be the primary map~\eqref{eq:primary}.
\begin{enumerate}[label=(\alph*)]
\item $\tau_P$ is continuous on $(0,\rho_B)$, with $\tau_P\to0$ as $\rho_f\to0$ and $\tau_P\to\tau_Z$ as $\rho_f\to\rho_B$.
\item $\Pi$ maps $\Sigma_P$ diffeomorphically onto an open set $\Pset_+^\circ\subset\Cset_+$, whose boundary is the union of the closed segment of $M_+$ from the apex to $B_+$, the axis segment $\{0\}\times[0,y_E]$, the closure of $\Dpt^+$ and the arc $\Lambda_+=\Bbar_+[0,\tau_Z]$.
\item Let $\Pset_+$ be $\Pset_+^\circ$ together with the open segment $\UP_+$, the apex, the axis segment $\{0\}\times(0,y_E]$ and the open curve $\Dpt^+$, and let $\Tset_+=(K\cap\{x\ge0\})\setminus\Pset_+$. Then every point of $\UP_+\cup\{\mathbf 0\}$ and of the axis segment below $E$ has a neighbourhood in $K\cap\{x\ge0\}$ contained in $\Pset_+$, $\Tset_+$ contains $\Lambda_+$, the $\NUP$ and the axis above $E$, and the closure of $\Tset_+$ is $\Tset_+\cup\Dpt^+\cup\{E\}$.
\end{enumerate}
\end{lemma}
\begin{proof}
(a) $\tau_D$ is continuous with $\tau_D(0^+)=0$ by \PI{Proposition~6(c)}, and $\tau_J$ is smooth on $(\rho_E,\rho_B)$ by Proposition~\ref{prop:F1}(iv). The function $F$ is $C^1$ near $(\rho_E,\tau_*)$, because $\Wt$ is $C^1$ (Proposition~\ref{prop:trib}), with $F(\rho_E,\tau_*)=0$ and $F_\tau<0$; the implicit function theorem gives a $C^1$ root branch through $(\rho_E,\tau_*)$, which coincides with $\tau_J$ for $\rho_f>\rho_E$ by uniqueness of the root, so $\tau_J(\rho_f)\to\tau_*=\tau_D(\rho_E)$. The same argument at $(\rho_B,\tau_Z)$, where the primary map extends continuously, gives $\tau_J\to\tau_Z$.

(b) $\Sigma_P$ lies in the chart parameter set, on which $\Pi$ is an injective local diffeomorphism (Lemma~\ref{lem:prim}), so $\Pset_+^\circ=\Pi(\Sigma_P)$ is open and $\Pi$ is a diffeomorphism onto it. If $\x\in\partial\Pset_+^\circ$, then $\x=\lim\Pi(p_n)$ with $p_n\in\Sigma_P$; the $p_n$ are bounded and have a limit point $p\in\overline{\Sigma}_P$, and $p\notin\Sigma_P$ because $\Pset_+^\circ$ is open and disjoint from its boundary. So $\x\in\Pi(\partial\Sigma_P)$. By (a), $\partial\Sigma_P$ consists of the edge $\tau=0$, $0\le\rho_f\le\rho_B$, mapped onto the segment of $M_+$ from the apex to $B_+$; the edge $\rho_f=\rho_B$, $0\le\tau\le\tau_Z$, mapped onto $\Lambda_+$; and the graph of $\tau_P$, mapped onto the axis contacts $\{0\}\times[0,y_E]$ for $\rho_f\le\rho_E$ (their heights increase from $0$ to $y_E$, \PI{Proposition~6(c)}) and onto $\Dpt^+$ for $\rho_f>\rho_E$. Conversely each of these points is a limit of points of $\Pset_+^\circ$.

(c) Near a point of $\UP_+$ or of the axis below $E$, the chart covers a relative neighbourhood in $K\cap\{x\ge0\}$ (the chart extends across these edges, \PI{Lemma~2(a)}), and by (b) the part of it in $K\cap\{x\ge0\}$ lies in $\Pset_+^\circ$ or on the edge itself; near the apex the same holds because $\tau_P(0^+)=0$. The remaining statements follow from (b) and the definition of $\Pset_+$.
\end{proof}

\begin{lemma}[Covering]\label{lem:cover}
Assume~\eqref{eq:F1}, and let $\Tset_+$ be as in Lemma~\ref{lem:region}. Then $s>y_E$ on $\Tset_+$ and $s\ge y_E$ on its closure, with equality only at $E\in\Pset_+$. Consequently $\Wt>W_E$ on $\Tset_+$, $\Wt\ge W_E$ on its closure with equality only at $E$, and $\Wt\ge W_E+(y_I-y_E)/(\nu-1)$ on $\Bbar_+[0,\tau_I]$.
\end{lemma}
\begin{proof}
\emph{Where the minimum lies.} The coordinates $(R,\sigma)$ are smooth on $\{x>0,y>0\}$ and $s=R-\nu\sigma$ is linear in them, so $\nabla s\ne0$ there. Since $s\ge R-\nu\pi/2\to\infty$ at infinity, $s$ attains its minimum over the closure $\overline\Tset_+$ at some point $p$. If $p$ lay in the open quadrant and in the interior of $\overline\Tset_+$, it would be a critical point; so $p$ is on the axis, on $M_+$, or on the boundary of $\overline\Tset_+$ inside the open quadrant. By Lemma~\ref{lem:region}, the points of $\overline\Tset_+$ on the axis are those with $y\ge y_E$, those on $M_+$ form the $\NUP$ together with $B_+$, and the remaining boundary points lie on $\Dpt^+$ or $\Lambda_+$.

\emph{Axis, $\Dpt^+$ and barrier.} On the axis $\sigma=0$ and $s=y\ge y_E$. On $\Dpt^+$ and on $\Bbar_+[0,\tau_I]$, $s>y_E$ by Proposition~\ref{prop:F1}(v).

\emph{$\NUP$.} On $M_+$, with $\phi=\alpha-\sigma$, the identities after~\eqref{eq:tribinv} give $\rho\sin\phi=1-\cos\sigma$ and $\rho\cos\phi=R-\sin\sigma$, so $0\le\sigma\le\alpha$, and
\[
R=\frac{\cos\phi-\cos\alpha}{\sin\phi},\qquad \rho=\frac{1-\cos(\alpha-\phi)}{\sin\phi},\qquad \frac{ds}{d\rho}=\cos\phi-\frac{\nu-1}{R}\sin\phi .
\]
On the $\NUP$, $\rho\ge\rho_B$, i.e.\ $\nu-1\le\rho+\sin\alpha-1$, and a direct computation gives
\begin{equation}\label{eq:NUPid}
R\cos\phi-(\rho+\sin\alpha-1)\sin\phi=\frac{(1-\sin\phi)(1-\cos\alpha\cos\phi)}{\sin\phi}>0
\end{equation}
for $0<\phi<\alpha$. Hence $ds/d\rho>0$ on the $\NUP$, for every $\nu>1$ and $\alpha$, and $s$ increases along the $\NUP$ from $s(B_+)>y_E$.

\emph{Values.} With $s\ge y_E$, $(\nu-1)(\Wt-W_E)=R-\sigma-y_E=s-y_E+(\nu-1)\sigma\ge0$, with equality iff $\sigma=0$ and $s=y_E$, i.e.\ at $E$. On $\Bbar_+[0,\tau_I]$, $s\ge y_I$ gives the last bound.
\end{proof}

\subsection{The synthesis and the main results}
The right half of $K$ is the disjoint union of the sets $\Pset_+$ and $\Tset_+$ of Lemma~\ref{lem:region}. $\Tset_+$ contains the barrier arc $\Lambda_+$, the $\NUP$, the axis above $E$ and the \emph{channel} $\Cset_+\cap\Tset_+$ between $\Dpt^+$ and the axis. Let $\Pset$, $\Tset$ be the unions with the mirror images and define
\begin{equation}\label{eq:W}
W=W_P\ \text{on}\ \Pset,\qquad W=\Wt\ \text{on}\ \Tset,\qquad W=0\ \text{on}\ \UP\cup\{\mathbf 0\},
\end{equation}
where on $\Pset_-$ the left branch $W_P(-x,y)$ is used. The \emph{barrier} is $\Bbar=\Bbar_+(0,\tau_Z)\cup\Bbar_-(0,\tau_Z)$, the part of the $\BUP$ characteristics across which $W$ jumps. Its \emph{lens side} is the side facing the $\UP$, covered by the primaries of $\Pset$, and its \emph{far side} the side into which $\mathbf n$ points; on the arcs from $Z_\pm$ to $I$ both sides carry $\Wt$, and these arcs, though semipermeable, carry no jump.

\begin{lemma}[Structure of $W$]\label{lem:P}
Assume~\eqref{eq:F1}. (a) $W$ is continuous on $K\setminus(\Bbar\cup\{B_+,B_-\})$. Across $\Bbar$ it jumps up, from the lens-side trace $\tau$ to $\Wt>\tau$, and the jump vanishes at $Z_\pm$, where $W$ is continuous. (b) $W$ is the lower envelope of the admissible branches: points of $\Pset$ with $\rho_f\ge\rho_E$ carry $W_P\le\Wt$, and points of the channel carry $\Wt<W_P$. (c) $W(B_\pm)=\Wt(B_\pm)>0$, while $W\to0$ as $B_\pm$ is approached from $\Pset$.
\end{lemma}
\begin{proof}
On $\Dpt^\pm$, $W_P=\Wt$ by definition, and at $E$ and $Z_\pm$ both values equal $W_E$ and $\tau_Z$. On $\Bbar$ the jump is $F(\rho_B,\tau)>0$ for $\tau<\tau_Z$, by Proposition~\ref{prop:F1}(iii), and the lens-side trace is $\tau$ by Lemma~2(b) of Part~I. Its proof writes $G_P(\tau')=-\int\sin|\tau'-\tau''|\,q(\tau'')\,d\tau''$ between $\tau_P$ and $\tau'$ and needs $G_P<0$ off $\tau_P$ on the interval containing the roots that define $W_P$; here these roots lie in $(0,\tau_I)$, and $G_P<0$ on $[0,\tau_I]\setminus\{\tau_P\}$ because $q>0$ on $[0,\tau_I)$. Part (b) is the sign of $F$ (Proposition~\ref{prop:F1}(ii),(iv)). For (c), $B_+\in\Tset_+$, so $W(B_+)=\Wt(B_+)>W_E$ by Lemma~\ref{lem:cover}, while the lens-side trace tends to $0$ at $B_+$ and $W=0$ on the adjacent $\UP_+$; this is the discontinuity at $B_\pm$ excluded in (a).
\end{proof}

\begin{figure}[!htbp]
\centering
\includegraphics[width=\textwidth]{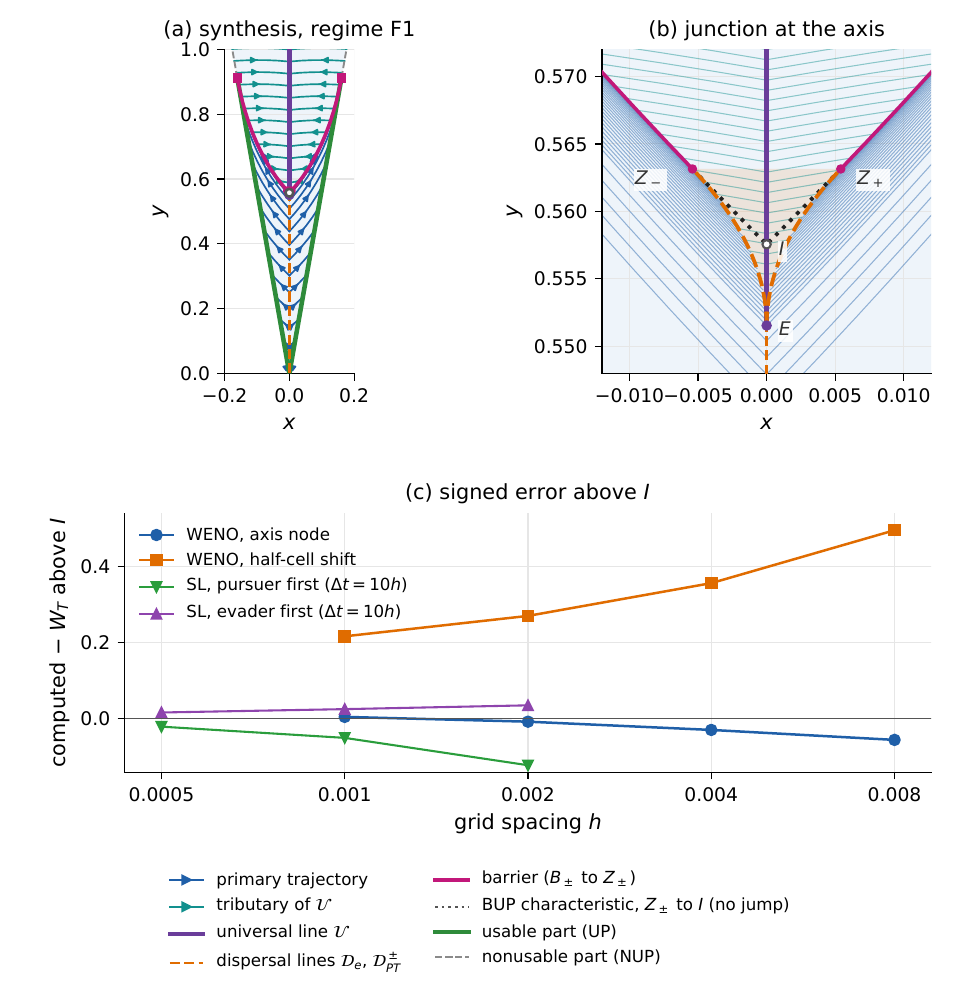}
\caption{Regime F1, $(\nu,\alpha)=(1.1,10^\circ)$. (a) The synthesis of Theorem~\ref{thm:value}. (b) Junction near the axis: $\Dpt^\pm$ bound the shaded tributary channel and meet the barrier at $Z_\pm$; the dotted continuations of the $\BUP$ characteristics reach $I$ without a jump. (c) Mean signed error against the proved value $W_T$ on the axis column $0.575\le y\le0.60$, for fixed-domain WENO computations with a grid node on the axis or a half-cell shift, and for semi-Lagrangian computations with either player first (Section~\ref{sec:num}).}
\label{fig:F1}
\end{figure}

The pursuer uses the relay of Part~I. The evader uses the primary feedback of Theorem~1(ii) of Part~I in $\Pset$, and in $\Tset$ it follows the tributary feedback until it can switch to a primary without loss. To make the switch precise, fix $\varepsilon>0$ and $\eta\in(0,\rho_B-\rho_E)$ so small that the strip $\Sigma_\eta=\{\x\in\Pset_+:\rho_f(\x)>\rho_B-\eta\}$ and its mirror image lie in $\{\Wt>W_E+c\}$, where $2c>0$ is the minimum of $\Wt-W_E$ on the arc $\Lambda_+=\Bbar_+[0,\tau_Z]$ (Lemma~\ref{lem:cover}); this is possible because the arcs of $\Pset_+$ with $\rho_f$ near $\rho_B$ end before $\tau_J(\rho_f)<\tau_Z$ and converge uniformly to $\Lambda_+$ (Lemma~\ref{lem:region}(a)). Let
\begin{equation}\label{eq:Aeps}
A_\varepsilon=\bigl\{\x\in\overline{\Cset}_\pm:\ \rho_f(\x)\le\rho_B-\eta,\ W_P(\x)\le\Wt(\x)+\varepsilon\bigr\},
\end{equation}
with $\rho_f$ and $W_P$ taken in each half with that half's branch. Since $\sigma\ge0$, $\Wt\le W_E+(R-y_E)/(\nu-1)$, and $W_E<y_E/(\nu-1)$ (Section~\ref{sec:univ}); hence $\Wt<-\varepsilon\le W_P-\varepsilon$ on the neighbourhood $\{R<y_E-(\nu-1)(W_E+\varepsilon)\}$ of the apex, which therefore does not meet $A_\varepsilon$ for $\varepsilon<y_E/(\nu-1)-W_E$, the range assumed from now on. So $A_\varepsilon$ is compact, being defined by a continuous inequality on the image of the compact strip $0\le\rho_f\le\rho_B-\eta$, $0\le\tau\le\tau_D(\rho_f)$ with a neighbourhood of the apex removed; it contains $\Pset\cap\{\rho_E\le\rho_f\le\rho_B-\eta\}$ and $E$.

\begin{theorem}[Feedback verification in F1]\label{thm:feedback}
Assume~\eqref{eq:F1}, and let $W$ be given by~\eqref{eq:W}.
\begin{enumerate}[label=(\roman*)]
\item \emph{Pursuer's guarantee.} If the pursuer plays $v_p=-1$, $\omega_p=-\sgn x$, with Filippov sliding on $x=0$, then $t_f\ge W(\x_0)$ for every $\x_0\in K$ and every measurable evader control.
\item \emph{Evader's guarantee.} If $\x_0\in\Pset$ and the evader plays the primary feedback $\mathbf u=\nu(C,-S)^\top$, $\theta=\alpha+W_P(\x)$, of the half containing $\x_0$ (the right one on the axis), then $t_f\le W(\x_0)$ for every measurable pursuer control. If $\x_0\in\Tset$ and the evader plays $\mathbf u_T$ until the first time the state enters $A_\varepsilon$, and from then on the primary feedback of the half containing the state at that time, then $t_f\le W(\x_0)+\varepsilon$ for every measurable pursuer control.
\end{enumerate}
Both feedbacks define nonanticipative strategies; hence $V_E\le W\le V_P$ on $K$.
\end{theorem}

\begin{theorem}[Value of the game in F1]\label{thm:value}
Assume~\eqref{eq:F1}. Then $V_E=V_P=W$ at every state of $K$, the barrier included. Each player can guarantee $W$ to arbitrary precision by a sampled state feedback that does not observe the opponent's current control.
\end{theorem}

In contrast to \PI{Theorem~2}, there is no exceptional state. The barrier arcs end at $Z_\pm$, off the axis, so a sampled relay never switches near them, and the axis points $E$ and $I$ carry no jump. The value is discontinuous only across $\Bbar$ and at its end points $B_\pm$, where it is positive although $B_\pm\in\partial K$, because from the far side of the barrier the evader cannot reach the thin usable part next to $B_\pm$. In particular, in regime F1 the lower envelope of the admissible branches (Lemma~\ref{lem:P}(b)) is the value of the game.

\subsection{Proof of Theorem~\ref{thm:feedback}}
\begin{lemma}[Primary feedback for $\nu>1$]\label{lem:prifb}
Assume~\eqref{eq:F1}. From any point of $\Cset_+$ or of its edges on $\UP_+$ and on the axis, the right-branch primary feedback gives $t_f\le W_P(\x_0)$ against every pursuer control. Along the trajectory $\rho_f$ does not increase, and the state does not return to the axis.
\end{lemma}
\begin{proof}
This is \PI{Theorem~1(ii)}, whose proof uses $\nu\le1$ in two places, both restored by $\nu S<1$ on $\Cset_+$ (Lemma~\ref{lem:prim}). First, in (28) of Part~I the coefficient of $\Delta_\omega$ is $gC-q(\nu-S)=(C-q)(\nu-1/S)-xC/S$, where $C-q=\int_0^{W}(\nu-\sin(\alpha+s))\,ds\ge0$, $\nu-1/S<0$ and $x\ge0$; so $\dot\rho_f\le0$. Second, on the axis $\dot x\ge\nu C-y=[g+S(1-\nu S)]/C>0$, so the state enters $\{x>0\}$ and does not return. The decrease $\dot W_P\le-1$ of (27) of Part~I uses only $H^*=0$ and $\psi>0$, and the remaining steps do not use $\nu$.
\end{proof}

\paragraph{The pursuer} Let $\x(\cdot)$ be a relay solution \PI{Lemma~1}; for $\nu>1$ it is defined while $y>0$, i.e.\ in $K\setminus\{\mathbf 0\}$. Let $\Lambda_+=\Bbar_+[0,\tau_Z]$, the closed barrier arc. Since $x$ decreases along $\Bbar_+$, $\Lambda_+\subset\{x\ge x_Z\}$ with $x_Z>0$. Near $\Lambda_+$ let $D$ be the signed distance to the curve formed by $\Bbar_+[0,\tau_I)$ and, near $B_+$, the $\NUP$, positive on the side into which $\mathbf n$ points (the $\Tset$ side). This curve is $C^{1,1}$: it is analytic along $\Bbar_+$, which is regular on $[0,\tau_I)$ because $q>0$ there, and it is tangent to the $\NUP$ at $B_+$ \PI{Eq.~(19)}.

\begin{lemma}[Local invariance near the barrier]\label{lem:inv}
There are a neighbourhood $N$ of $\Lambda_+$ in $\{x>x_Z/2\}$ and $L>0$ such that, whenever the pursuer plays the constant control $(-1,-1)$ while the state stays in $N\cap K$, $\frac{d}{dt}D(\x(t))\ge-L|D(\x(t))|$ a.e., whatever the evader does. Hence $D\ge0$ and $D>0$ are preserved there; in particular the state does not cross $\Bbar_+$ near $\Lambda_+$ from the side $D>0$, and from $\{D\ge0\}$ near $B_+$ it does not leave $K$.
\end{lemma}
\begin{proof}
This is \PI{Lemma~3(b)}, case $P\ne c$. The function $m(\x)=\min_{\|\mathbf u\|\le\nu}\nabla D^\top\dot\x(\x,\mathbf u,-1,-1)$ is Lipschitz, vanishes on $\Bbar_+$ by Proposition~\ref{prop:barfast}(iv), and is at least $\rho-\rho_B\ge0$ on the $\NUP$ \PI{Eq.~(8)}. So $m(\x)\ge-L|D(\x)|$, and Gr\"onwall's inequality gives the preservation. The region $\{D\ge0\}$ near $B_+$ lies between $\Bbar_+$ and the $\NUP$, inside $K$.
\end{proof}

Near $\Lambda_+$, $W$ has a convenient representation. Let $\widehat W_P=W_P$ on $\Cset_+\cup\UP_+$, with $W_P=0$ on $\UP_+$, the continuous extension of the chart value \PI{Lemma~2(a)}, and $\widehat W_P=+\infty$ elsewhere. On a neighbourhood $N'\subset N$ of $\Lambda_+$,
\begin{equation}\label{eq:minrep}
W=\min\bigl(\widehat W_P,\Wt\bigr),
\end{equation}
because on the lens side and on $\UP_+$ near $B_+$, $W=W_P\le\Wt$, in the channel near $Z_+$ $\Wt<W_P$, and beyond the barrier, on it and at $B_+$, $\widehat W_P=+\infty$ (Lemma~\ref{lem:P}).

\begin{lemma}[$W$ along relay solutions]\label{lem:mono}
Along every relay solution starting in $K$, $\varphi(t)=W(\x(t))+t$ is nondecreasing while $\x(t)\in K$, and $W(\x(t))\to0$ as $t\uparrow t_f$ if $t_f<\infty$. Hence $t_f\ge W(\x_0)$.
\end{lemma}
\begin{proof}
A function on an interval is nondecreasing if every $t$ has a right neighbourhood on which it is $\ge\varphi(t)$ and a left neighbourhood on which it is $\le\varphi(t)$. (If $\varphi(a)>\varphi(b)$, $a<b$, then one of the two properties fails at $c=\sup\{s\in[a,b]:\varphi(s)\ge\varphi(a)\}$.) We verify this local property according to the position of $\x(t)$.

(i) $\x(t)\notin\Lambda=\Lambda_+\cup\Lambda_-$. If $\x(t)\notin\overline\Pset$, then $W=\Wt$ near $\x(t)$, and Proposition~\ref{prop:trib}(a) gives the local property. Otherwise let $a(x,y)=W_P(|x|,y)$, the primary value with the branch of each side; it is defined, continuous and locally Lipschitz on a neighbourhood of $\overline{\Pset}\setminus\Lambda$, because there $\rho_f$ stays away from $\rho_B$ and the chart extends across the axis and the $\UP$ \PI{Lemma~2(a)}. Choose a convex neighbourhood $V$ of $\x(t)$ that does not meet $\Lambda$ and on which $a$ and $\Wt$ are $L$-Lipschitz. On $V\cap K$, which is convex, $W$ is continuous (Lemma~\ref{lem:P}(a)), equals $a$ on the relatively closed set $V\cap\overline\Pset$ and $\Wt$ on the relatively closed set $V\cap\overline\Tset$, and these two sets cover $V\cap K$.

\emph{Gluing.} A continuous function on a convex set that is $L$-Lipschitz on each of finitely many relatively closed sets covering it is $L$-Lipschitz. Indeed, on a segment $[\mathbf p,\mathbf q]$ every point $\mathbf z$ has a neighbourhood that meets only the sets containing $\mathbf z$, since the others are closed; for $\mathbf w$ in that neighbourhood, $\mathbf w$ and $\mathbf z$ lie in a common set, so $|W(\mathbf w)-W(\mathbf z)|\le L\|\mathbf w-\mathbf z\|$. The restriction of $W$ to the segment therefore has Dini derivatives bounded by $L\|\mathbf q-\mathbf p\|$, and hence is Lipschitz with that constant. So $W$ is $L$-Lipschitz on $V\cap K$, and $W\circ\x$ is absolutely continuous while $\x$ stays in $V$.

\emph{Derivative.} Let $S_P$ and $S_T$ be the sets of times at which $\x(s)$ lies in $\overline\Pset$ or in $\overline\Tset$. On $S_P$, $W\circ\x=a\circ\x$, so at almost every $s\in S_P$ (every density point at which both are differentiable) their derivatives agree; likewise with $\Wt$ on $S_T$. It remains to show that $a\circ\x$ and $\Wt\circ\x$ decrease at most at unit rate at almost every time of $S_P$ and $S_T$, respectively. For $\Wt$ this is Proposition~\ref{prop:trib}(a). For $a$: off the axis, $a$ coincides near $\x(s)$ with the branch of that side, which satisfies $H^*=0$ with the relay as saddle control, so $\frac{d}{ds}a\ge-1$. At almost every time with $x(s)=0$ the solution slides, $\dot x=0$, and $\frac{d}{ds}a(\x(s))=\partial_ya\,\dot y=\lambda_y(u_y+1)$. Points of $\overline\Pset$ on the axis have $y\le y_E$, where $\lambda_y\le1/(\nu-1)$ (Proposition~\ref{prop:E}), so $\frac{d}{ds}a\ge-\lambda_y(\nu-1)\ge-1$. Hence $\varphi$ is nondecreasing near $t$, and in particular has the local property.

(ii) $\x(t)\in\Lambda_+$, including $B_+$ and $Z_+$. Use~\eqref{eq:minrep}. $\Wt(\x(\cdot))+s$ has the local property by Proposition~\ref{prop:trib}(a). So does $\widehat W_P(\x(\cdot))+s$: on $\Cset_+\cup\UP_+$, where $W_P$ is continuous and satisfies $H^*=0$ with the relay as saddle control, it decreases at most at unit rate; leaving $\Cset_+$ across the $\BUP$ curve makes it jump to $+\infty$; and by Lemma~\ref{lem:inv} the state cannot enter $\Cset_+$ near $\Lambda_+$, whose boundary there is the $\BUP$ curve, crossed from its side $D>0$. A minimum of two functions with the local property has it.

\emph{Exit.} The relay solution does not leave $K$ through the $\NUP$ \PI{Lemma~3(a)} or near $B_\pm$ from $\{D\ge0\}$ (Lemma~\ref{lem:inv}), so it leaves through $\UP\cup\{\mathbf 0\}$, approached from $\Pset$. There $W=W_P$ near the exit point, and $W_P$ extends continuously by $0$ to $\UP\cup\{\mathbf 0\}$ \PI{Lemma~2(a)}, so $W\to0$. Then $t_f=\lim_{t\uparrow t_f}\varphi(t)\ge\varphi(0)=W(\x_0)$.
\end{proof}

\paragraph{The evader} For $\x_0\in\Pset$ the claim is Lemma~\ref{lem:prifb}. For $\x_0\in\Tset$ we need a separation property. Let $\Omega_0$ be the union of the image of $\{0<\rho_f<\rho_E,\ 0<\tau<\tau_D(\rho_f)\}$, its mirror image, the axis segment $\{0\}\times(0,y_E)$, the segments of $\UP_\pm$ with $\rho<\rho_E$ and the apex. Let $\Gamma=\Gamma_E^+\cup\Gamma_E^-$, where $\Gamma_E^+$ is the $\rho_E$-primary from $\UP_+$ to $E$.

\begin{lemma}[Separation]\label{lem:sep}
(a) $\Omega_0\cap\Tset=\emptyset$, $\Omega_0$ is open relative to $K$, and its boundary relative to $K$ is $\Gamma$. (b) $\Gamma\subset A_\varepsilon$, and there is $\delta>0$ such that $\{\x\in\Omega_0:\operatorname{dist}(\x,\Gamma)<\delta\}\subset A_\varepsilon$. (c) If a path in $K$ with speed at most $M$ starts outside $\Omega_0$ and later lies in $\Omega_0$ at distance $\ge\delta$ from $\Gamma$, it lies in $A_\varepsilon$ on a time interval of length $\ge\delta/M$ before that.
\end{lemma}
\begin{proof}
(a) By Lemma~\ref{lem:prim}, $\Omega_0\cap\{x\ge0\}$ is the image of the open parameter set under an injective local diffeomorphism, whose edges are $\UP_+\cap\{\rho<\rho_E\}$ ($\tau=0$), the apex ($\rho_f\to0$), $\Gamma_E^+$ ($\rho_f\to\rho_E$) and the axis segment below $E$ ($\tau\to\tau_D$). The edges other than $\Gamma_E^+$ belong to $\Omega_0$ and have relative neighbourhoods in $K$ contained in $\Omega_0$ (on the axis after gluing the mirror image; on $\UP$ and at the apex because $K$ ends there). (b) On $\Gamma_E^+$, $\rho_f=\rho_E<\rho_B-\eta$ and $F(\rho_E,\tau)\ge0$ for $\tau\le\tau_*$, i.e.\ $W_P\le\Wt$. The functions $\rho_f$, $W_P$ and $\Wt$ are uniformly continuous on a compact neighbourhood of $\Gamma_E^+$ in $\overline{\Cset}_+$ on which $\rho_f$ stays below $\rho_B-\eta$, which gives $\delta$; the left half is symmetric. (c) Let $t_2$ be the first time with $\x\in\Omega_0$ and $\operatorname{dist}(\x,\Gamma)\ge\delta$, and $t_1=\sup\{t<t_2:\x(t)\notin\Omega_0\}$. By (a), $\x(t_1)\in\Gamma$, and on $(t_1,t_2)$ the path is in $\Omega_0$ within $\delta$ of $\Gamma$, hence in $A_\varepsilon$, and covers a distance $\ge\delta$.
\end{proof}

By construction,
\begin{equation}\label{eq:Pdecomp}
\Pset\setminus A_\varepsilon\subset\Sigma^\pm_\eta\cup\Omega_0 .
\end{equation}
Let $\x_0\in\Tset$. (1) Before the switch, $\Wt(\x(t))\le\Wt(\x_0)-t$ by~\eqref{eq:WTdot}, which holds on all of $\{y>0\}$; $\mathbf u_T$ is Lipschitz, so the closed loop is well posed. (2) The switch occurs by $t^*=\Wt(\x_0)-W_E$ unless the state has left $K$. Indeed, $\Wt(\x(t^*))\le W_E$, so $\x(t^*)\notin\Tset\cup\Sigma^\pm_\eta$ by Lemma~\ref{lem:cover} and the choice of $\eta$. As $E\in A_\varepsilon$, \eqref{eq:Pdecomp} leaves only $\x(t^*)\in\Omega_0\setminus A_\varepsilon$, at distance $\ge\delta$ from $\Gamma$ by Lemma~\ref{lem:sep}(b); since $\x_0\notin\Omega_0$, Lemma~\ref{lem:sep}(c) shows that the path met $A_\varepsilon$ earlier. (3) After the switch at $t_1$, $\x(t_1)\in A_\varepsilon\subset\overline\Cset_\pm$, and Lemma~\ref{lem:prifb} gives $t_f\le t_1+W_P(\x(t_1))\le t_1+\Wt(\x(t_1))+\varepsilon\le\Wt(\x_0)+\varepsilon$. Exits through the $\NUP$ before the switch only shorten $t_f$.

The relay is nonanticipative by Lemma~1 of Part~I. The switching evader is nonanticipative because the closed loop has a unique solution for each pursuer control and the switching time is the hitting time of a closed set by that solution. As in Corollary~1 of Part~I, this gives $V_E\le W\le V_P$. \qed

\begin{remark}[The boundary of F1]\label{rem:equality}
On $\nu\sin\theta_c=1$ the $\BUP$ characteristic reaches the axis exactly at its cusp, $I=P_c$, and $q(\tau_I)=0$. The proofs above never use $q(\tau_I)>0$: they use $q>0$ on $[0,\tau_I)$, on the chart, whose retro-times stay below $\tau_D(\rho_f)<\tau_I$, and on the barrier arc, which ends at $\tau_Z<\tau_I$. Theorems~\ref{thm:feedback} and~\ref{thm:value} therefore hold on the closed region~\eqref{eq:F1}. At $I$ the value is the smooth tributary value, so the cusp of the $\BUP$ characteristic leaves no trace in $W$.
\end{remark}

\section{Regime F2}\label{sec:F2}
In regime F2, $\nu\sin\theta_c>1$, the $\BUP$ characteristic ends at its cusp $P_c$ inside the open right half of $K$ (Proposition~\ref{prop:barfast}(iii)), and $q$ changes sign before some primaries reach the axis. Primaries launched sufficiently close to $B_+$ then bend back toward the ray before reaching the front axis, so that distinct arcs of the family overlap (Fig.~\ref{fig:fast}). The Jacobian $\rho_B-\rho_f$ does not vanish, so this is not a focal phenomenon; the overlap comes from arcs that turn back after their minimal bearing, and we call the part of such an arc beyond its minimal bearing its \emph{folded} part.

\subsection{Where the jump ends}\label{sec:F2split}
The geometric classification of Proposition~\ref{prop:barfast}(iii) does not decide which branch carries the value near the cusp. Let $F(\rho_B,\tau)=\Wt(\Bbar_+(\tau))-\tau$ as in Section~\ref{sec:F1}. The formula for $F_\tau$ in Proposition~\ref{prop:F1}(ii) holds along every primary, and on $[0,\tau_c]$ we have $\alpha\le\theta\le\theta_c<\pi/2$ and $0\le\sigma<\pi/2$, so $F(\rho_B,\cdot)$ decreases strictly on $[0,\tau_c]$; the proof of $F(\rho_B,0)>0$ given there uses only Propositions~\ref{prop:E} and~\ref{prop:trib}. Hence exactly one of the following holds, apart from the equality case $\Wt(P_c)=\tau_c$, which is treated with F2b in Section~\ref{sec:F2b}.
\begin{itemize}
\item \emph{F2a}: $\Wt(P_c)<\tau_c$. Then $F(\rho_B,\cdot)$ has a unique root $\tau_Z\in(0,\tau_c)$, and $\Wt<\tau$ on $\Bbar_+(\tau_Z,\tau_c]$.
\item \emph{F2b}: $\Wt(P_c)>\tau_c$. Then $\Wt>\tau$ on $\Bbar_+[0,\tau_c]$.
\end{itemize}
The two cases are separated by a single curve.

\begin{proposition}[The second threshold]\label{prop:alphaZ}
For every $\nu>1$, $\Delta(\alpha)=\Wt(P_c)-\tau_c$ increases strictly on $(\alpha_F(\nu),\pi/2)$, from a negative value to the positive limit $\bigl(\sqrt{\nu^2-1}-\arccos(1/\nu)\bigr)/(\nu-1)$. Hence it has exactly one root $\alpha_Z(\nu)$, F2a is $\alpha_F(\nu)<\alpha<\alpha_Z(\nu)$, and F2b is $\alpha_Z(\nu)<\alpha<\pi/2$.
\end{proposition}
\begin{proof}
At $P_c$, $q=0$, so the identities $\nu\sin\delta+q\cos\delta=1$ and $\nu\cos\delta-q\sin\delta=R$ of Proposition~\ref{prop:F1}(v), $\delta=\theta-\sigma$, give $R_c=\sqrt{\nu^2-1}$ and $\sigma_c=\theta_c-\beta$ with $\beta=\arcsin(1/\nu)$. By~\eqref{eq:WT}, $\Delta=\tau_*+(R_c-y_E-\sigma_c)/(\nu-1)-\tau_c$. Write $\theta_*=\alpha+\tau_*$, $S_*=\sin\theta_*$, $C_*=\cos\theta_*$. Equation~\eqref{eq:fE} reads $\cos\alpha+\nu\alpha-\nu\theta_*+\nu C_*S_*/(1+S_*)-C_*=0$; its derivative in $\theta_*$ is $-S_*K_*$ with $K_*=(\nu-1)+\nu/(1+S_*)$, and its derivative in $\alpha$ is $\nu-\sin\alpha$, so $d\theta_*/d\alpha=(\nu-\sin\alpha)/(S_*K_*)$, while $dy_E/d\alpha=-\nu/(1+S_*)\,d\theta_*/d\alpha$. With $\tau_c=\cos\alpha/\nu$ and $\sigma_c=\alpha+\tau_c-\beta$ this gives
\[
\Delta'(\alpha)=\frac{(\nu-\sin\alpha)(1-S_*)}{(\nu-1)S_*}>0 .
\]
At $\alpha_F$, $\Delta<0$ by Proposition~\ref{prop:F1}(i). As $\alpha\uparrow\pi/2$, $\tau_*$, $y_E$ and $\tau_c$ tend to $0$, which gives the limit; it is positive because $\sqrt{\nu^2-1}=\tan\gamma>\gamma=\arccos(1/\nu)$.
\end{proof}

For example, $\alpha_Z=40.3^\circ$, $26.8^\circ$, $20.6^\circ$, $14.7^\circ$, $7.8^\circ$ and $4.0^\circ$ for $\nu=1.01$, $1.05$, $1.1$, $1.2$, $1.5$ and $2$ (Fig.~\ref{fig:regimes}). Since $\Delta$ is continuous and negative at the boundary of F1, F2a is a band of F2 next to F1. In F2a the tributary through $Z_+=\Bbar_+(\tau_Z)$ is genuine, $s(Z_+)>y_E$ (Section~\ref{sec:F2a}, part (c) of the proof of Theorem~\ref{thm:F2a}). For example, $(\nu,\alpha)=(1.1,16^\circ)$ lies in F2, with $\nu\sin\theta_c=1.0054$, but $\Wt(P_c)=0.68401<\tau_c=0.87387$, with $s(P_c)-y_E=0.0129$; there $\tau_Z=0.73652$ and $Z_+=(0.01450,0.45125)$.

In F2a, therefore, the jump across $\Bbar_\pm$ ends at $Z_\pm$ before the cusp, as in F1. By Lemma~\ref{lem:barG} with $q(\tau_c)=0$, $G(\tau_c+\delta;P_c)=\nu(\delta-\sin\delta)$, which vanishes only at $\delta=0$, so every primary through a point near $P_c$ has retro-time near $\tau_c$, above $\Wt$ there; the tributary carries $W$ near $P_c$, and the cusp leaves no trace in it, as on the boundary of F1 (Remark~\ref{rem:equality}). Section~\ref{sec:F2a} shows that in F2a the synthesis of Section~\ref{sec:F1} is again the value of the game. Section~\ref{sec:F2b} concerns F2b, which contains all F2 examples used below: $\Wt(P_c)-\tau_c=0.468$, $0.549$ and $0.414$ for $(1.5,30^\circ)$, $(1.5,60^\circ)$ and $(2,15^\circ)$.

Table~\ref{tab:regimes} summarizes the regimes for $\nu>1$ and where each is solved.

\begin{table}[!htbp]
\centering\small
\caption{Regimes of the faster-evader game, $\nu>1$, $0<\alpha<\pi/2$. In every regime the pursuer's optimal strategy is the relay $(v_p,\omega_p)=(-1,-\sgn x)$, and each player can guarantee $W$ to arbitrary precision by sampled state feedback.}
\label{tab:regimes}
\renewcommand{\arraystretch}{1.15}\setlength{\tabcolsep}{4pt}
\begin{tabularx}{\linewidth}{@{}l>{\raggedright\arraybackslash}X>{\raggedright\arraybackslash}X>{\raggedright\arraybackslash}X l@{}}
\toprule
Regime & Condition & Jump across the $\BUP$ characteristic & Evader's strategy & Result\\
\midrule
F1 & $\nu\sin\theta_c\le1$ ($\alpha\le\alpha_F$) & ends at $Z_\pm$, before the axis contact $I$ & tributary feedback, one switch to the primary feedback & Thms.~\ref{thm:feedback}, \ref{thm:value}\\
F2a & $\nu\sin\theta_c>1$, $\Wt(P_c)<\tau_c$ ($\alpha_F<\alpha<\alpha_Z$) & ends at $Z_\pm$, before the cusp & as in F1 & Thm.~\ref{thm:F2a}\\
F2b & $\Wt(P_c)>\tau_c$ ($\alpha>\alpha_Z$) & runs to the cusps $P_{c\pm}$; folded primaries on the far side near them & committed primary heading, or tributary feedback then a committed core heading & Thm.~\ref{thm:F2b}\\
$\alpha=\alpha_Z$ & $\Wt(P_c)=\tau_c$ & runs to the cusps; tributary on the far side near them & as in F2b & Thm.~\ref{thm:F2b}\\
\bottomrule
\end{tabularx}
\end{table}

\subsection{The value in regime F2a}\label{sec:F2a}
In F2a we keep the definitions of Section~\ref{sec:F1} with $\tau_I$ replaced by $\tau_c$: for $0<\rho_f<\rho_B$ let $\tau_m(\rho_f)=\min(\tau_D(\rho_f),\tau_c)$, where $\tau_D=\infty$ if the primary does not reach the axis on $[0,\tau_c]$; the chart $\Cset_+$ is the image of $\{0<\rho_f<\rho_B,\ 0<\tau<\tau_m(\rho_f)\}$; and $\Pset_\pm$, $\Tset_\pm$, $\Dpt^\pm$, $Z_\pm$, $\Lambda_\pm$, $W$ and $A_\varepsilon$ are defined as in Section~\ref{sec:F1}.

\begin{theorem}[Value in regime F2a]\label{thm:F2a}
In regime F2a the conclusions of Theorems~\ref{thm:feedback} and~\ref{thm:value} hold: the relay guarantees $t_f\ge W$, the switching evader guarantees $t_f\le W+\varepsilon$, and $V_E=V_P=W$ at every state of $K$, the barrier included, each player being able to guarantee $W$ to arbitrary precision by sampled state feedback.
\end{theorem}
\begin{proof}
Section~\ref{sec:F1} and Appendix~\ref{app:value} use~\eqref{eq:F1} only in four places: the chart (Lemma~\ref{lem:prim}), the junction (Proposition~\ref{prop:F1}(i), (iii), (iv)), the barrier part of the covering (Proposition~\ref{prop:F1}(v) and Lemma~\ref{lem:cover}), and the sign $\nu S<1$ in the primary feedback (Lemmas~\ref{lem:prifb} and~\ref{lem:unif}). The other arguments use only $q>0$ on the retro-times involved, the regularity and semipermeability of $\Bbar_+$ near $\Lambda_+$ (Proposition~\ref{prop:barfast}), Propositions~\ref{prop:E} and~\ref{prop:trib}, and the geometry near $E$, and they go through verbatim with $\tau_I$ replaced by $\tau_c$. We replace the four ingredients; throughout, $q>0$ on $[0,\tau_c)$ and $\theta\le\theta_c<\pi/2$.

\emph{(a) Junction.} The formula for $F_\tau$ in Proposition~\ref{prop:F1}(ii) gives $F_\tau<0$ on the arcs of $[0,\tau_c]$ before their axis contact, for every $\rho_f$, and $\tau_Z\in(0,\tau_c)$ is the root of Section~\ref{sec:F2split}. First, $\tau_*<\tau_c$. Otherwise the $\rho_E$-primary is still in $\{x\ge0\}$ at $\tau_c$, and $F(\rho_E,\tau_c)\ge F(\rho_E,\tau_*)=0$; the points $\x(\rho_f,\tau_c)=P_c-(\rho_B-\rho_f)(\sin\theta_c,\cos\theta_c)^\top$, $\rho_E\le\rho_f\le\rho_B$, form a segment in $\{x\ge0,y>0\}$ along which $\partial F/\partial\rho_f=\cos(\theta_c-\sigma)/(\nu-1)>0$ by~\eqref{eq:gradT}, so $F(\rho_B,\tau_c)>0$, contrary to F2a. Let $\rho_E<\rho_f<\rho_B$; $F(\rho_f,0)>0$ as in Proposition~\ref{prop:F1}. If $\tau_D(\rho_f)\le\tau_c$, the axis contact lies above $E$ and $F<0$ there, by the integration in Proposition~\ref{prop:F1}(i): the contacts at retro-times $\tau\in[\tau_*,\tau_c]$ have $g=\nu-S-y_DC>0$, since $g=0$ would place $\Bbar_+$ on the axis before its cusp, contrary to Proposition~\ref{prop:barfast}(iii), and $\lambda_y>1/(\nu-1)$ because $f<0$ there. If $\tau_D(\rho_f)>\tau_c$, then $\x(\rho_f,\tau_c)=P_c-(\rho_B-\rho_f)(\sin\theta_c,\cos\theta_c)^\top$ lies on the segment from $P_c$ toward the axis point $(0,\cot\theta_c)$, along which $\Wt$ decreases at the rate $\cos(\theta_c-\sigma)/(\nu-1)>0$ away from $P_c$, by~\eqref{eq:gradT}; hence $F(\rho_f,\tau_c)<F(\rho_B,\tau_c)<0$. In both cases $F(\rho_f,\cdot)$ has a unique root $\tau_J(\rho_f)\in(0,\tau_m(\rho_f))$, and the rest of Proposition~\ref{prop:F1}(iv) and the first part of (v) follow as before. By the formula for $d\tau_J/d\rho_f$ in the proof of (v), $\tau_J$ increases from $\tau_*$ to $\tau_Z$, so $\tau<\tau_Z$ on $\Pset_+$.

\emph{(b) Chart.} On the parameter set of $\Cset_+$, $g>0$ and $q>0$, so $x_\tau=-gC-qS<0$ and each arc is a graph over $x$; at fixed $x$, $\partial y/\partial\rho_f=g/(gC+qS)>0$. For $x>0$, the labels whose arcs pass over $x$ form the interval $x/\sin\alpha<\rho_f<r(x)$, where $r(x)=\min(\rho_B,\bar r(x))$ and $\bar r(x)$ solves $x(\bar r,\tau_c)=x$, because $x(\rho_f,\tau_c)$ increases with $\rho_f$ and the arc covers $x$ exactly when $x(\rho_f,\tau_m)<x<\rho_f\sin\alpha$. The arcs are therefore strictly ordered at each $x$, so the primary map is injective on the parameter set, and, with Jacobian $g>0$, a diffeomorphism onto $\Cset_+$. Hence $W_P$ is smooth on $\Cset_+$ with $\nabla W_P=(-C,S)^\top/g$, the same ordering shows that the lens-side trace on $\Bbar_+$ is $\tau$, and Lemma~\ref{lem:region} and its proof hold with $\tau_D$ replaced by $\tau_m$. The chart extends across the axis and the $\UP$ as before.

\emph{(c) Covering.} On $\Dpt^+$, $s>y_E$ by (a) and Proposition~\ref{prop:F1}(v), and $s(Z_+)>y_E$ as the limit of the increasing function $s$ along $\Dpt^+$. By~\eqref{eq:sbar}, $s$ decreases along $\Bbar_+$ on $[0,\tau_Z]$, so $s>y_E$ on $\Lambda_+$, and in particular at $B_+$. The argument on the $\NUP$, \eqref{eq:NUPid}, does not use~\eqref{eq:F1}. Hence Lemma~\ref{lem:cover} holds, with its last bound replaced by $\Wt>W_E$ on $\Lambda_+$, which is all that the choice of $\Sigma_\eta$ uses.

\emph{(d) Primary feedback.} In F2a the sign $\nu S<1$ may fail on $\Pset_+$ near $Z_\pm$, and with it the monotonicity of $\rho_f$: the coefficient $(C-q)(\nu-1/S)-xC/S$ of $\Delta_\omega$ in $\dot\rho_f$ is then positive near the axis. Lemma~\ref{lem:prifb} uses the monotonicity only to keep the state in the chart, away from the axis and from the $\BUP$ curve, and two facts valid at every retro-time $\tau<\tau_c$ do this instead. Fix $\bar\tau\in(\tau_Z,\tau_c)$; for $\varepsilon$ small, $\tau\le\bar\tau$ on $\Pset_+$ and on $A_\varepsilon$, because $\tau_J(\rho_B-\eta)<\tau_Z$ and $F_\tau$ is bounded away from zero on compact sets. Under the primary feedback $\tau=W_P$ decreases, so $\tau\le\bar\tau$ along the trajectory.
\begin{itemize}
\item \emph{Axis.} At an axis contact with retro-time $\tau<\tau_c$,
\begin{equation}\label{eq:axismargin}
\nu C-y=\frac{g+S(1-\nu S)}{C}=\frac{q-C(1-\nu S)}{S}>0,
\end{equation}
by the first form if $\nu S\le1$, since $g>0$, and by the second if $\nu S>1$, since $q>0$. So $\dot x\ge\nu C-y>0$ there against every pursuer control, and the state enters $\{x>0\}$ and does not return to the axis.
\item \emph{Barrier.} By the decomposition of $\dot\rho_f$ in Lemma~\ref{lem:prifb}, $\dot\rho_f=-(C+qS/g)\Delta_v+[gC-q(\nu-S)]\Delta_\omega/g$, and both coefficients are negative when $g<g_0=(\nu-1)(\cos\alpha-\nu\bar\tau)$, because $q\ge\cos\alpha-\nu\bar\tau>0$, $\nu-S\ge\nu-1$ and $C\le1$. Hence $\rho_f$ does not increase while $\rho_f>\rho_B-g_0$, and $\rho_f\le\max(\rho_f(\x_0),\rho_B-g_0)<\rho_B$ along the trajectory.
\end{itemize}
With these two facts, the first-exit argument of Lemma~\ref{lem:prifb} applies on the compact part $\{\rho_f\le\max(\rho_f(\x_0),\rho_B-g_0),\ \tau\le\bar\tau\}$ of the closed chart, and $W_P$ decreases at least at unit rate, which uses only $H^*=0$ and $\psi>0$; so $t_f\le W_P(\x_0)$. In Lemma~\ref{lem:unif} (Appendix~\ref{app:value}; $\overline\Cset(\bar\rho)$, $C_0$ and $T_0$ are defined there) the strip $\overline\Cset(\bar\rho)$ is replaced by the image of $\{0\le\rho_f\le\hat\rho,\ 0\le\tau\le\min(\tau_D(\rho_f),\hat\tau)\}$, with $\bar\tau<\hat\tau<\tau_c$, $\tilde\rho=\max(\bar\rho,\rho_B-\hat g_0)$, $\hat g_0=(\nu-1)(\cos\alpha-\nu\hat\tau)$ and $\hat\rho=(\tilde\rho+\rho_B)/2$. The axis margin~\eqref{eq:axismargin} has a positive minimum on this compact set, and on $\tilde\rho\le\rho_f\le\hat\rho$ the held feedback gives $\dot\rho_f\le C_0h$, so the state needs a time $(\hat\rho-\tilde\rho)/(C_0h)>T_0$ to reach $\rho_f=\hat\rho$ once $h$ is small; the rest of the proof is unchanged.
\end{proof}

Theorem~\ref{thm:F2a} needs no description of the folded primaries: they pass only through $\Tset$, where the value is $\Wt$. On the arcs of $\Bbar_\pm$ from $Z_\pm$ to the cusps both sides carry $\Wt$, so, as on the boundary of F1, the cusp leaves no trace in the value. Together with Section~\ref{sec:F1}, the escape time is therefore the value of the game for every $\nu>1$ and $\alpha<\alpha_Z(\nu)$, that is, wherever the jump across the $\BUP$ characteristic ends before the characteristic reaches the axis or its cusp. The replacements in (d) apply in F1 as well and give a second proof of Lemma~\ref{lem:prifb} there.

\subsection{The value in regime F2b}\label{sec:F2b}
\paragraph{The synthesis} Proposition~\ref{prop:E}, Proposition~\ref{prop:trib} and the local tests hold in F2 without change. We take as $W$ the smallest value carried by the admissible branches at a point: primaries before or after their fold, retained up to their first axis contact or, for primaries that do not cross the axis, their first return to $M_+$ (Lemma~\ref{lem:retained}), and genuine tributaries. On the evader's dispersal line $\Dpt$ between a retained right primary and a right tributary, $H^*\ge0$ on the whole segment between the two one-sided gradients, which is the test for a minimum of two branches, because both one-sided gradients have $\lambda_y>0$ and $\Phi_\omega<0$: on the tributary by~\eqref{eq:gradT}, and on the retained primary because $\psi>0$ there, folded arcs included (Lemma~\ref{lem:retained}). The pursuer control $(-1,-1)$ maximizes the Hamiltonian on the whole segment between them, so $H^*$ is concave there and vanishes at both ends. The resulting structure consists of:
\begin{itemize}
\item the axis dispersal line $\De$ from the apex to $E$, and the universal line $\U$ above $E$;
\item the dispersal line $\Dpt$, which leaves $E$ tangentially to the axis and ends either on the far side of the barrier, at a point $Y_+$ where the far-side value passes from the tributary to the folded primary (e.g.\ $(\nu,\alpha)=(1.5,30^\circ)$ and $(2,15^\circ)$), or on the $\NUP$ (e.g.\ $(1.5,60^\circ)$, at $\rho\simeq1.12\,\rho_B$);
\item the barrier, the whole arc of $\Bbar_\pm$ from $B_\pm$ to the cusps $P_{c\pm}$, across which $W$ jumps; in F2b the lens trace $\tau$ stays below $\Wt$ along it, by the dichotomy above.
\end{itemize}
Primaries that return to $M_+$ after their fold are truncated at the return point, which lies on the $\UP$ for the $\BUP$ characteristic and may lie on the $\NUP$ for primaries launched farther from $B_+$.

\begin{figure}[!htbp]
\centering
\includegraphics[width=\textwidth]{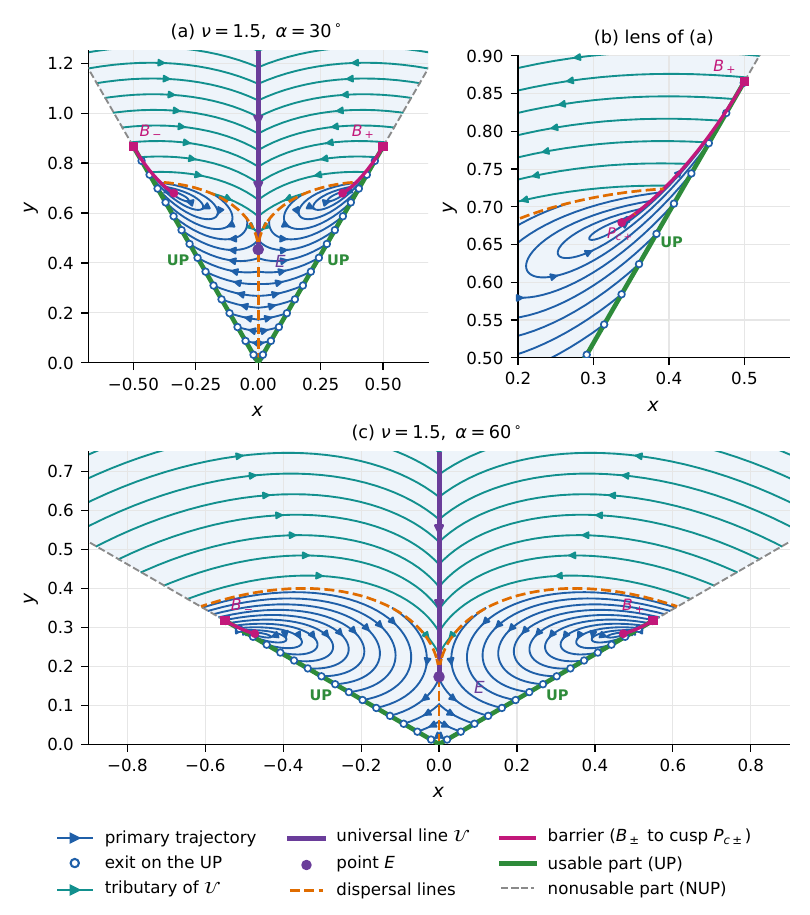}
\caption{Synthesis in regime F2b (Theorem~\ref{thm:F2b}). (a) $(\nu,\alpha)=(1.5,30^\circ)$; (b) enlargement of the lens next to the barrier; (c) $(1.5,60^\circ)$, with $\Dpt$ ending on the $\NUP$. Forward-time arrows: primaries blue, tributaries teal. Universal line $\U$: purple; dispersal lines: orange dashed; barrier: magenta, from $B_\pm$ (squares) to the cusps (dots); $\UP$: green; $\NUP$: gray dashed. Branches are retained where their values are smallest.}
\label{fig:fast}
\end{figure}

\paragraph{Values across the barrier and at its ends} On the side of $\Bbar_+$ facing the usable ray, the primary value tends to $\tau$. The far-side limit $W_{\rm out}$ is carried by tributaries or folded primaries and gives the jump $\Delta W=W_{\rm out}-\tau$. Near the cusp it is carried by the folded primary: as in F2a, every primary through a point near $P_c$ has retro-time near $\tau_c$, while $\Wt$ exceeds $\tau_c$ by a fixed margin there, and primaries of the left family stay in the left half. The folded primary gives $\Delta W=3\epsilon+O(\epsilon^3)$ at $\Bbar_+(\tau_c-\epsilon)$, i.e.\ square-root decay with the distance to $P_c$, while a transverse approach to $P_c$ from the far side has a cube-root law (Appendix~\ref{app:cusp}); the jump along the barrier is tabulated in \OR{S9}. As in F1, we set $W=W_{\rm out}$ on the barrier, the value of the side to which the pursuer can confine the state (Proposition~\ref{prop:barfast}(iv)). At $P_c$ both one-sided limits equal $\tau_c$, and the folded branch merges with the $\BUP$ characteristic; we set $W(P_{c\pm})=\tau_c$, which Theorem~\ref{thm:F2b} identifies as the value there.

\paragraph{Forward play} Figure~\ref{fig:fwd} follows one optimal play for $(1.5,30^\circ)$ from the tributary leaving $\U$ at $U_0=(0,0.95)^\top$, at retro-time $0.3$, with $W=1.7519$. The pursuer plays $(v_p,\omega_p)=(-1,-1)$ for $0.3$, $(-1,0)$ for $0.9923$ along $\U$ and $(-1,-1)$ for $0.4596$ along the primary from $E$, all ordinary Reeds--Shepp controls; integrating the global kinematics reproduces the reduced trajectories to $5\times10^{-10}$ (\OR{S6}).

\begin{figure}[!htbp]
\centering
\includegraphics[width=\textwidth]{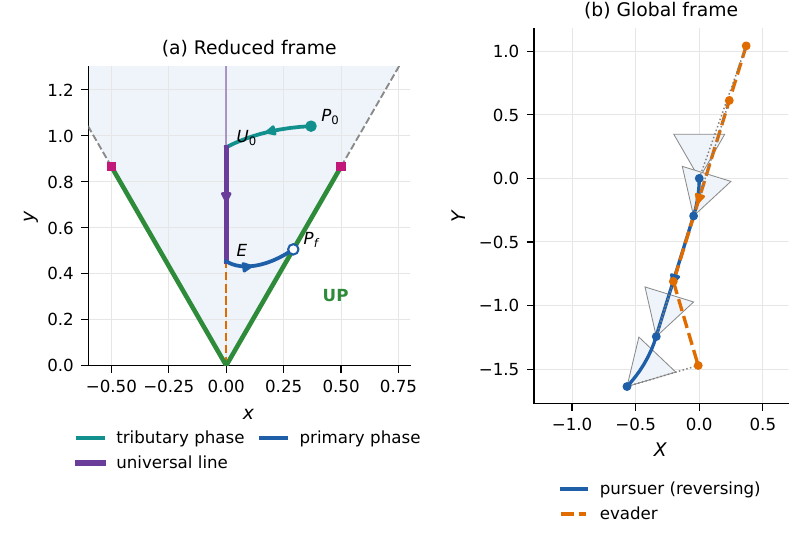}
\caption{An optimal play for $\nu=1.5$, $\alpha=30^\circ$, in forward time. (a) Reduced frame: tributary phase to the universal line at $U_0$ (teal), straight approach along $\U$ to $E$ (purple), and primary phase from $E$ to the escape point on the $\UP$ (blue). (b) Global frame: the reversing pursuer (solid) and the evader (dashed), with the line of sight and the field of view at the start, at $U_0$, at $E$ and at escape.}
\label{fig:fwd}
\end{figure}

\paragraph{The barrier up to its cusp} In F1 and F2a the barrier ends at $Z_\pm$, where $q>0$, and Lemma~\ref{lem:inv} uses a curvature bound for $\Bbar_+$. In F2b the curvature $1/q$ of $\Bbar_+$ is unbounded near $P_c$, and we use instead the scalar function~\eqref{eq:Gscalar}.

\begin{lemma}[Semipermeability up to the cusp]\label{lem:Mfun}
Let $\nu>1$, $0<\tau_0<\tau_c$ and $\mathbf b=\Bbar_+(\tau_0)$. There are a neighbourhood $U$ of $\mathbf b$ and $r>0$ such that, for $\x\in U$, $G(\cdot\,;\x)$ has a unique local maximizer $\tau^*(\x)$ in $(\tau_0-r,\tau_0+r)$, depending smoothly on $\x$. The function $M(\x)=G(\tau^*(\x);\x)$ vanishes on $\Bbar_+$, is negative on its far side and positive on its lens side in $U$, and $M(\x(t))$ is nonincreasing while $\x(t)\in U$ whenever the pursuer plays $(-1,-1)$, whatever the evader does.
\end{lemma}
\begin{proof}
By Lemma~\ref{lem:barG}, $G(\tau_0+\delta;\mathbf b)=\nu(\delta-\sin\delta)-q(\tau_0)(1-\cos\delta)$, which has a nondegenerate local maximum $0$ at $\delta=0$, since $q(\tau_0)>0$; the implicit function theorem gives $\tau^*$. Since $G_\tau(\tau^*;\x)=0$, \eqref{eq:rhof} gives $\rho_f(\theta^*,\x)=\rho_B$, $\theta^*=\alpha+\tau^*$, and $\nabla M=\partial_\x G(\tau^*;\x)=-\mathbf n(\theta^*)$. Hence $M$ decreases across $\Bbar_+$ in the direction of $\mathbf n$, which gives the signs. With $(v_p,\omega_p)=(-1,-1)$, \eqref{eq:reduced} and~\eqref{eq:rhof} give $\mathbf n(\theta^*)^\top\dot\x=\mathbf n^\top\mathbf u+(x+1)\sin\theta^*+y\cos\theta^*=\mathbf n^\top\mathbf u+\rho_B+\sin\alpha\ge0$, so $\frac{d}{dt}M=-\mathbf n(\theta^*)^\top\dot\x\le0$.
\end{proof}

No curvature bound enters: $U$ shrinks as $\tau_0\to\tau_c$, but a crossing from the far side to the lens happens at a single barrier point, so the relay prevents every such crossing other than at $P_c$. At $P_c$ the envelope $W$ is continuous: $G(\tau_c+\delta;P_c)=\nu(\delta-\sin\delta)$ vanishes only at $\delta=0$, so all primary values near $P_c$ are near $\tau_c$, while $\Wt$ exceeds $\tau_c$ by a fixed margin. Near the tip the scalar function~\eqref{eq:Gscalar} also has an exact meaning. With $z=(x+1,y)^\top$, the control $(-1,-1)$ turns~\eqref{eq:reduced} into $\dot z=Az+\mathbf u$, where $A$ generates rotations, so the largest value of $\ell=x\cos\alpha-y\sin\alpha$, the signed distance to the line of $M_+$, that the evader can reach at time $T$ is
\begin{equation}\label{eq:Gsupport}
\sup_{\|\mathbf u\|\le\nu}\ell(\x(T))=G(T;\x_0),
\end{equation}
attained by $\mathbf u(s)=\nu e^{-A(T-s)}(\cos\alpha,-\sin\alpha)^\top$ (\OR{S9}). At the cusp, $G(T;P_c)=\nu\bigl((T-\tau_c)-\sin(T-\tau_c)\bigr)<0$ for $0\le T<\tau_c$: against $(-1,-1)$, which the relay plays near $P_c$, an evader starting at $P_c$ cannot even reach the line of $M_+$ before $\tau_c$. In particular, no manoeuvre from $P_c$ reaches a right-primary state $\x$ at time $T$ with $T+W_P(\x)<\tau_c$, since following that primary would then exit through $M_+$ before $\tau_c$. This certificate is local: it concerns the constant control $(-1,-1)$ and exits through the right ray, not trajectories on which the relay changes side.

\paragraph{Verification} We now show that $W$ is the value in F2b. The argument does not need a global description of the dispersal lines between unfolded primaries, folded primaries and tributaries, whose topology differs between the examples. It rests on four ingredients: an explicit parameter domain for all retained primaries (Lemma~\ref{lem:retained}); an evader control that commits in advance to the heading of a primary and guarantees its time against every pursuer control, folded arcs included (Lemma~\ref{lem:commit}); an auxiliary envelope in which the tributary formula~\eqref{eq:WT} is used without the condition $s\ge y_E$, for which both guarantees are proved first (Propositions~\ref{prop:auxub} and~\ref{prop:relayF2b}); and a covering argument that identifies this envelope with $W$ (Proposition~\ref{prop:coverF2b}). Throughout, $b=\rho_f+\sin\alpha$, so that $g=\nu-b$ and, by~\eqref{eq:primary}, the right primary is $\x(\rho_f,\tau)=(bS-1+qC,\ bC-qS)^\top$, and $d=y\sin\alpha-x\cos\alpha\ge0$ is the distance to the line of $M_+$; along a primary, $d=\cos\alpha-b\sin\tau-q\cos\tau$ and $d_\tau=g\cos\tau+q\sin\tau$. A primary is \emph{retained} at retro-time $T$ if its arc on $[0,T]$ stays in the closed right half of $K$, i.e.\ it is admissible in the sense above.

\begin{remark}[Two verification arguments]\label{rem:twoproofs}
The argument below is not specific to F2b. Committed headings, the auxiliary envelope and the covering argument could also verify $W$ in F1 and F2a; there they would need a version of Lemma~\ref{lem:retained} for primaries that reach the axis before folding, and the treatment of the end $Z_\pm$ of the jump given by Lemma~\ref{lem:inv}. Numerically, the auxiliary envelope~\eqref{eq:auxV} defined below coincides with $W$ in F1 and F2a as well (\OR{S9}). We keep the proof of Section~\ref{sec:F1} for those regimes because it gives more: the evader's guarantee is a state feedback, the tributary feedback with a single switch to a primary, with sampled versions (Theorem~\ref{thm:value}), whereas a committed heading is open-loop once chosen. In F2b the feedback argument breaks down near the continuation of the barrier, where a folded branch need not persist under feedback, and the committed heading avoids this question.
\end{remark}

\begin{lemma}[Retained primaries]\label{lem:retained}
Let $\nu\sin\theta_c>1$.
\begin{enumerate}[label=(\alph*)]
\item A primary retained up to $T$ has $T<\pi/2$, $S>0$ and $\psi=xS+yC>0$ on $[0,T]$; if it reaches $\theta=\pi/2$, then $b>1$.
\item $A(\tau)=q-C+\nu CS$ has a unique zero $\tau_a$ on $(0,\pi/2)$, and $\tau_c<\tau_a<\pi/2-\alpha$. With $S_a=\sin(\alpha+\tau_a)$, $C_a=\cos(\alpha+\tau_a)$, $b_a=S_a+\nu C_a^2$ and $\rho_a=b_a-\sin\alpha$, one has $\rho_E<\rho_a<\rho_B$, $b_a>1$ and $\tau_*<\tau_a$.
\item For $0<\rho_f<\rho_a$ the primary is retained up to its first axis contact $\tau_D(\rho_f)<\tau_a$ and no further; for $\rho_a\le\rho_f<\rho_B$ it never crosses the axis and is retained up to its first return to $M_+$, at a retro-time below $\pi/2$. Hence a root $T$ of~\eqref{eq:Gscalar} at a state of the right half of $K$, with label $\rho_f$ from~\eqref{eq:rhof}, is retained if and only if $0\le T<\pi/2$, $0<\rho_f<\rho_B$, and [$T\le\tau_a$ or $\rho_f\ge\rho_a$]. The primary $\rho_a$ touches the axis at $(0,\nu C_a)$ at $\tau_a$ and stays in the right half; its continuation beyond the tangency is admitted in the auxiliary construction below but is never selected (Lemma~\ref{lem:axiscomp}).
\end{enumerate}
\end{lemma}
\begin{proof}
(a) At $\tau_h=\pi/2-\alpha$, $x=b-1$ and $q(\tau_h)=\cos\alpha-\nu\tau_h<0$, because $\pi/2-\alpha>\cos\alpha$. An arc retained up to $\tau_h$ needs $b\ge1$, and $b=1$ is excluded because then $x=S-1+qC<0$ just before $\tau_h$. At $\tau=\pi/2$ the formula for $d$ gives $d=\cos\alpha-b<0$, so a retained arc ends before $\pi/2$; hence $S>0$. Before $\theta=\pi/2$, $\psi=xS+yC>0$ in the right quadrant; afterwards $\psi=b-S\ge b-1>0$.
(b) The condition $x\ge0$ reads $b\ge a(\tau)=(1-qC)/S$, and $a'=A/S^2$, $A'=S(1-2\nu S)$, so $A$ has at most one interior maximum on $[0,\tau_h]$. Since $A(0)=\nu\cos\alpha\sin\alpha>0$, $A(\tau_c)=C_c(\nu S_c-1)>0$, with $S_c=\sin\theta_c$ and $C_c=\cos\theta_c$, and $A(\tau_h)=q(\tau_h)<0$, $A$ has a unique zero there, after $\tau_c$; on $[\tau_h,\pi/2]$, $A'$ can change sign only from negative to positive and $A(\pi/2)=\cos\alpha+\sin\alpha-\nu(\pi/2+\sin\alpha\cos\alpha)<0$, so $A<0$ there. Thus $a$ has its unique maximum $a(\tau_a)=b_a$, where $q=C_a-\nu C_aS_a$. Since $S_a>S_c>1/\nu$, $\nu-b_a=S_a(\nu S_a-1)>0$ and $b_a-1=(1-S_a)(\nu(1+S_a)-1)>0$. In~\eqref{eq:fE}, $f(\tau_a)=-\nu C_aS_a^2/(1+S_a)<0$, so $\tau_*<\tau_a$, and $\rho_E<\rho_a$ because $a$ increases on $[0,\tau_a]$.
(c) As $a(0)=\sin\alpha$, a label with $\sin\alpha<b<b_a$ has exactly one first axis contact on the increasing branch of $a$; before it, $x>0$, $y>0$ (the contact height $y_D=(C-q)/S$ is positive because $(q-C)'=S-\nu<0$), and $d>0$, because $d_\tau/\sin\tau=g\cot\tau+q$ decreases, so $d$ increases and then decreases, and $d>0$ at the contact. After the contact $x<0$ until $a$ falls below $b$ again, beyond $\tau_a$. For $b\ge b_a$, $x\ge0$ on $(0,\pi/2)$, and $d(\pi/2)=\cos\alpha-b<0$ with the same single maximum gives exactly one return to the ray. The test follows, because the state itself supplies $x\ge0$ and $d\ge0$ at $T$.
\end{proof}

Points of the axis carry, besides tributaries, only primaries that end there, at their axis contact.

\begin{lemma}[Axis comparison]\label{lem:axiscomp}
Along the axis contacts, $y_D=(C-q)/S$ increases on $[0,\tau_a]$ from $0$ to $\nu C_a$, and
\[
F_D(\tau)=W_E+\frac{y_D(\tau)-y_E}{\nu-1}-\tau\le0\qquad(0\le\tau\le\tau_a),
\]
with equality only at $\tau_*$. Along the primary $\rho_a$ after its tangency with the axis at $\tau_a$, $\Wt<\tau$.
\end{lemma}
\begin{proof}
At a contact $g=\nu-S-y_DC>0$ by Lemma~\ref{lem:retained}, and $y_D'=g/S>0$. Differentiation gives $F_D'=Cf(\tau)/((\nu-1)S^2)$, so $F_D$ has its unique maximum at $\tau_*$, where it vanishes. At $\tau_a$, $F_D<0$, and by Proposition~\ref{prop:F1}(ii) $\Wt-\tau$ does not increase along the primary $\rho_a$ while it stays in the right quadrant.
\end{proof}

\begin{lemma}[Committed heading]\label{lem:commit}
Let $\x_0=\x(\rho_f,T)$ with $0<\rho_f<\rho_B$ be a retained primary state. If the evader plays the heading prescribed in the reduced frame, $\mathbf u(t)=\nu\bigl(\cos(\alpha+T-t),-\sin(\alpha+T-t)\bigr)^\top$ on $[0,T]$ and then $\nu(\cos\alpha,-\sin\alpha)^\top$, then $t_f\le T$ against every measurable pursuer control.
\end{lemma}
\begin{proof}
Let $\x^0(t)=\x(\rho_f,T-t)$ be the play against $(-1,-1)$, and write $\Delta_v=v_p+1\ge0$, $\Delta_\omega=\omega_p+1\ge0$, $\mathbf e=\x-\x^0$. Then $\dot{\mathbf e}=A_0\mathbf e+(\Delta_\omega y,\ -\Delta_v-\Delta_\omega x)^\top$ with $A_0=\bigl(\begin{smallmatrix}0&-1\\1&0\end{smallmatrix}\bigr)$, so
\[
x(t)-x^0(t)=\int_0^t\bigl[\cos(t-r)\Delta_\omega y+\sin(t-r)(\Delta_v+\Delta_\omega x)\bigr]dr .
\]
The kernels are nonnegative because $T<\pi/2$ (Lemma~\ref{lem:retained}). While the state is in $K$, $y\ge0$, so $x\ge x^0-2\int_0^t x^-\,dr$ with $x^-=\max(-x,0)$; since $x^0\ge0$, Gr\"onwall's inequality gives $x^-=0$, and then $x\ge x^0\ge0$. Let $\theta(t)=\alpha+T-t$ and $\psi=x\sin\theta+y\cos\theta$. While $\theta<\pi/2$, $\psi\ge0$, with equality only at the apex. While $\theta\ge\pi/2$, which happens first and only if $b>1$, $h=(x+1)\sin\theta+y\cos\theta$ satisfies $\dot h=-\Delta_v\cos\theta+\Delta_\omega(y\sin\theta-x\cos\theta)\ge0$ and $h(0)=b$, so $\psi=h-\sin\theta\ge b-1>0$. Now let $L(t)=G(T-t;\x(t))$. Differentiating~\eqref{eq:Gscalar} along~\eqref{eq:reduced} gives
\begin{equation}\label{eq:Lcert}
\dot L=\Delta_v\sin\theta+\Delta_\omega\psi\ge0 .
\end{equation}
Since $L(0)=0$ and $L(T)=x(T)\cos\alpha-y(T)\sin\alpha$, the state is on or beyond the line of $M_+$ at $T$ unless it has left $K$ before. If it is still in $K$ at $T$, then $L(T)=0$, so $\dot L=0$ a.e., which forces $\Delta_v=\Delta_\omega=0$ a.e., and the state is at the usable point $\rho_f(\sin\alpha,\cos\alpha)$, where the outward speed under $\nu(\cos\alpha,-\sin\alpha)$ is at least $g>0$ against every pursuer control; so strict exit points occur arbitrarily soon after $T$.
\end{proof}

The same certificate gives the evader's guarantee at the cusps. From $P_c$, play the heading of Lemma~\ref{lem:commit} with $T=\tau_c+\delta$, $0<\delta<\pi/2-\theta_c$: the play against $(-1,-1)$ stays, for small $\delta$, near the barrier arc from $P_c$ to $B_+$, where $x\ge x(P_c)>0$, so the comparison applies and $\psi>0$ because $\alpha+T<\pi/2$; by Lemma~\ref{lem:barG}, $L(0)=G(\tau_c+\delta;P_c)=\nu(\delta-\sin\delta)>0$, so $L(T)>0$ and the state has left $K$ before $T$. Hence the evader guarantees $\tau_c+\varepsilon$ from $P_{c\pm}$ for every $\varepsilon>0$.

\paragraph{The core and the auxiliary envelope} Let $\Omega_0$ be the open region of $K$ swept by the primaries with $0<\rho_f<\rho_E$ before their axis contacts, with their mirror images, the axis below $E$, the segments of $\UP_\pm$ with $\rho<\rho_E$ and the apex, as in Section~\ref{sec:F1}; its boundary $\Gamma_E$ consists of the $\rho_E$-primaries from $\UP_\pm$ to $E$. On these arcs $q$ may change sign, since $\tau_*$ may exceed $\tau_c$, but $x_\tau/S=-q-g\cot\theta$ increases ($\theta<\pi/2$) and is negative at the axis contact, so $x_\tau<0$, the arcs are graphs over $x$ ordered by $\rho_f$, and $\Omega_0$ carries a single primary value $P_0$, locally Lipschitz up to $\Gamma_E$ and extended evenly. For $\x=\x(\rho_f,T)\in\Omega_0$ and $0<\eta\le T$,
\[
G(T-\eta;\x)=-g\eta-b(\eta-\sin\eta)-q(T)(1-\cos\eta)<0,
\]
directly if $q(T)\ge0$, and otherwise because $x_\tau(T)<0$ gives $-q(T)<g\cot(\alpha+T)$ and $\cot(\alpha+T)(1-\cos\eta)<\eta$; so $P_0$ is the earliest root of~\eqref{eq:Gscalar} and every other admissible primary through $\Omega_0$ takes longer. On $\Gamma_E$, $P_0\le\Wt$, with equality only at $E$, by the sign of $F_\tau$ in Proposition~\ref{prop:F1}(ii). Moreover $\Wt\ge m>0$ on $K\setminus\Omega_0$: $\Wt\to\infty$ at infinity, $\partial_y\Wt=\cos\sigma/(\nu-1)>0$ excludes interior minima, $\Wt\ge W_E$ on the axis above $E$, $\Wt$ increases along $M_+$ beyond $\rho_E$, and $\Wt>P_0\ge0$ on $\Gamma_E\setminus\{E\}$. Let $p(\x)$ be the smallest time of the retained primaries through $\x$ (Lemma~\ref{lem:retained}), together with the continuation of the primary $\rho_a$ beyond its axis tangency, with $p=+\infty$ if there is none, and set
\begin{equation}\label{eq:auxV}
V=P_0\ \text{on}\ \Omega_0,\qquad V=\min(p,\Wt)\ \text{on}\ K\setminus(\Omega_0\cup\{P_{c\pm}\}),\qquad V(P_{c\pm})=\tau_c,
\end{equation}
with $V=0$ on $\UP\cup\{\mathbf 0\}$. No condition $s\ge y_E$ is imposed on $\Wt$.

\begin{proposition}[Evader's bound]\label{prop:auxub}
From every state the evader guarantees $V+\varepsilon$ for every $\varepsilon>0$, and $V$ itself except at the cusps.
\end{proposition}
\begin{proof}
If $V$ is a primary time, use Lemma~\ref{lem:commit}; at the cusps, use the preceding paragraph. If $V=\Wt$ outside $\Omega_0$, play $\mathbf u_T$ until the state leaves $K$ or first meets $\Gamma_E$, at $t_1$. Before $t_1$, $\Wt(\x(t))\le\Wt(\x_0)-t$ by~\eqref{eq:WTdot}, and $\Wt\ge m$ outside $\Omega_0$, so $t_1\le\Wt(\x_0)$. At $t_1$ commit to the core primary through $\x(t_1)$: $t_f\le t_1+P_0(\x(t_1))\le t_1+\Wt(\x(t_1))\le\Wt(\x_0)$.
\end{proof}

\begin{lemma}[No downward edge]\label{lem:noedge}
$V$ is locally Lipschitz on $K$ away from the closed barrier arcs $\Bbar_\pm[0,\tau_c]$ and continuous at $P_{c\pm}$; across the barriers it jumps up from the lens trace $\tau$ to the far-side value, including at $B_\pm$.
\end{lemma}
\begin{proof}
At a fixed state with $x\ge0$, $G_T=\nu-r\sin(\alpha+T+\beta)$, with $r=\|(x+1,y)\|$ and $\tan\beta=y/(x+1)$, has at most two zeros on $(0,\pi/2)$, so there are finitely many roots, and each root with $g>0$ continues smoothly. A branch can therefore appear or disappear only on the images of the edges of the parameter domain of Lemma~\ref{lem:retained}, or where two roots merge, $g=0$, i.e.\ on the $\BUP$ curve. (i) Launch points and ray returns lie on $\partial K$. (ii) Axis contacts: below $E$ the core branch meets its mirror image; above $E$, Lemma~\ref{lem:axiscomp} gives $\Wt<\tau_D$, so the ending branch is not the minimum. (iii) The continuation of the primary $\rho_a$: there $\Wt<\tau$ outside $\Omega_0$ by Lemma~\ref{lem:axiscomp}, and inside $\Omega_0$ the time exceeds $\tau_a>\tau_*\ge P_0$. (iv) The $\BUP$ curve beyond the cusp, $\mathbf b=\Bbar_+(\tau_c+d)$, $d>0$: Lemma~\ref{lem:barG} gives $G(\tau_c;\mathbf b)=\nu(\sin d-d\cos d)>0\ge G(0;\mathbf b)$, so there is an earlier up-crossing root $T_0<\tau_c$ (or $\mathbf b$ is on the $\UP$, where the branch of time zero applies); it has $g>0$, since a root with $g=0$ before $\tau_c$ is a local maximum of $G$, and $\rho_f>0$, because $a$ increases on $[0,\tau_c]$ from $a(0)=\sin\alpha$ and $x\ge0$ gives $b\ge a(T_0)>\sin\alpha$; as $T_0<\tau_c<\tau_a$, it is retained by Lemma~\ref{lem:retained}(c). So the pair of roots that merges on the continuation lies above $V$ nearby. (v) The barrier, $\mathbf b=\Bbar_+(\tau_0)$, $0<\tau_0<\tau_c$: $G(T;\mathbf b)=\nu(T-\tau_0-\sin(T-\tau_0))-q(\tau_0)(1-\cos(T-\tau_0))<0$ for $T<\tau_0$, so every other root is later than $\tau_0$, and $\Wt(\mathbf b)>\tau_0$, because $F(\rho_B,\cdot)$ decreases strictly on $[0,\tau_c]$ and $F(\rho_B,\tau_c)\ge0$; the lens branch exists only where $M>0$ (Lemma~\ref{lem:Mfun}), and the jump is upward. At $B_+$ the same holds with the barrier continued by the $\NUP$. On $\Gamma_E$, $P_0$ and the outside branches agree with~\eqref{eq:auxV} and are continuous. A finite minimum of locally Lipschitz branches is locally Lipschitz. At $P_c$, $G(T;P_c)=\nu((T-\tau_c)-\sin(T-\tau_c))$ vanishes only at $\tau_c$, so all roots in $[0,\pi/2]$ at states near $P_c$ tend to $\tau_c$, and their labels tend to $\rho_B$. Fix $\eta>0$ with $\tau_c+\eta<\pi/2$; the signs $G(\tau_c-\eta;P_c)<0<G(\tau_c+\eta;P_c)$ persist near $P_c$, so each nearby state has a sign-changing root in $(\tau_c-\eta,\tau_c+\eta)$. At a root with $G_T=0$, $G_{TT}=-(G+q)=-q$, so a sign-changing multiple root needs $G=G_T=G_{TT}=0$, which holds only at $P_c$; hence every nearby state other than $P_c$ has an up-crossing root with $g>0$. Its label lies in $(\rho_a,\rho_B)$ near $P_c$, so it is retained by Lemma~\ref{lem:retained}(c). With $\Wt(P_c)\ge\tau_c$ this gives $V\to\tau_c$.
\end{proof}

\begin{proposition}[Pursuer's bound]\label{prop:relayF2b}
Along every relay solution, $V(\x(t))+t$ is nondecreasing until exit; hence the relay guarantees $t_f\ge V(\x_0)$.
\end{proposition}
\begin{proof}
On every retained primary $S>0$ and $\psi>0$ (Lemma~\ref{lem:retained}), so the Hamiltonian computation of Section~\ref{sec:F1} applies on folded arcs too: each primary branch decreases at most at unit rate under the relay, and so does $\Wt$ by Proposition~\ref{prop:trib}(a). On the axis, $V(0,y)$ is $P_0$ below $E$ and $W_U$ above, with slope at most $1/(\nu-1)$ (Proposition~\ref{prop:E}), and a sliding relay solution has $\dot y\ge-(\nu-1)$. With Lemma~\ref{lem:noedge}, the gluing argument of Lemma~\ref{lem:mono}(i) gives the local property away from the barriers and the cusps. Near a barrier point before the cusp, write $V$ as the minimum of the far-side branches and the lens branch extended by $+\infty$; by Lemma~\ref{lem:Mfun} the state cannot enter the lens from $M\le0$, so the argument of Lemma~\ref{lem:mono}(ii) applies, and near $B_\pm$ it applies verbatim. At the cusps $V$ is continuous, and on the closed set of times at which $\x(t)=P_{c\pm}$, $V(\x(t))+t=\tau_c+t$; continuity joins the intervals between such times. Exits through the $\NUP$ or at $B_\pm$ are excluded as in Section~\ref{sec:F1}, and at an exit through $\UP\cup\{\mathbf 0\}$, $V\to0$.
\end{proof}

\begin{proposition}[Covering]\label{prop:coverF2b}
If $\x\notin\Omega_0$ and $V(\x)=\Wt(\x)$, then $s(\x)\ge y_E$. Consequently $V=W$ on $K$.
\end{proposition}
\begin{proof}
Let the evader play $\mathbf u_T$ against the relay from $\x$. Then $\Wt(\x(t))+t=\Wt(\x)=V(\x)$ by~\eqref{eq:WTdot}, $V(\x(t))+t\ge V(\x)$ by Proposition~\ref{prop:relayF2b}, and $V\le\Wt$ outside $\Omega_0$, so $V=\Wt$ along the play until it meets $\Gamma_E$ or leaves $K$, which happens in finite time because $\Wt\ge m$. It cannot meet $\Gamma_E\setminus\{E\}$, where $P_0<\Wt$, nor leave through $\UP$, where $V=0<\Wt$; the $\NUP$, $B_\pm$ and the apex are excluded as before. So the play reaches $E$, along the axis from above. Off the axis, $\dot R=-\nu$ and $\dot\sigma=-1$, so $s=R-\nu\sigma$ is constant until the first axis contact, whose height is $s$; that contact is not below $E$, since the axis below $E$ lies in $\Omega_0$. Hence $s\ge y_E$. Conversely, every retained primary has time at least $V$ (by definition, and on $\Omega_0$ by the comparison above), and every genuine tributary value is at least $V$, because its play against the relay, followed by the primary from $E$, exits at time $\Wt$, so Proposition~\ref{prop:relayF2b} gives $V\le\Wt$. The continuation of the primary $\rho_a$ is never selected (Lemma~\ref{lem:axiscomp}). So $V$ is the lower envelope of retained primaries and genuine tributaries, which is $W$.
\end{proof}

\begin{theorem}[Value in regime F2b]\label{thm:F2b}
In regime F2b and on its boundary $\alpha=\alpha_Z(\nu)$, $V_E=V_P=W$ at every state of $K$, including the barriers and the cusps. The relay guarantees $t_f\ge W$; the evader guarantees $W$, except at the cusps, where it guarantees $\tau_c+\varepsilon$ for every $\varepsilon>0$, by committing to a primary heading, or by the tributary feedback until $\Gamma_E$ followed by a committed core heading. Each player can guarantee $W$ to arbitrary precision by sampled state feedback.
\end{theorem}
\begin{proof}
By Propositions~\ref{prop:auxub}, \ref{prop:relayF2b} and~\ref{prop:coverF2b}, $V_E\le W\le V_P$. The reverse inequalities follow from sampled strategies (Appendix~\ref{app:value}, Lemma~\ref{lem:sampF2b}).
\end{proof}

The committed heading is open-loop once chosen, so it needs no sampling, and it sidesteps the question of whether a folded branch persists under a feedback near the continuation of the barrier. Together with Lemma~\ref{lem:Mfun}, it also settles the cusp: the relay guarantees $\tau_c$ from $P_{c\pm}$ (Proposition~\ref{prop:relayF2b}), and the evader guarantees $\tau_c+\varepsilon$. The committed heading itself was integrated from $144$ retained primary states, $50$ of them folded, for $12$ parameter pairs, against six pursuers including one that minimizes $\dot L$ at each step, and from the cusps with fourth-order steps; every run left $K$ by $T$, within one integration step (\OR{S9}). Separately, as an empirical test of the synthesis rather than of the proved strategy, an evader that follows at every step the branch achieving the minimum was played from $17$ states around the cusp and the barrier for each of the three F2b examples, against greedy and adversarial pursuers; its overshoot stayed below $5\times10^{-4}$, and the relay was never beaten by more than $10^{-5}$ (\OR{S9}, Table~S15).

\paragraph{The second threshold} On $\alpha=\alpha_Z(\nu)$, $\Wt(P_c)=\tau_c$. The proofs above use the F2b inequality only through $\Wt>\tau$ on the barrier before the cusp, which also holds here because $F(\rho_B,\cdot)$ decreases strictly to $F(\rho_B,\tau_c)=0$, and through the continuity of $V$ at $P_c$, which needs only $\Wt(P_c)\ge\tau_c$; so Theorem~\ref{thm:F2b} includes this curve. The local picture at the cusp is different. As in the proof of Theorem~\ref{thm:F2a}(a), $\tau_*<\tau_c$, the junction points of $\Dpt^+$ tend to $P_c$, and integrating $ds/d\rho_f=(\nu-1)/R$ along them gives $s(P_c)>y_E$, so the tributary is genuine near $P_c$. Since $\nabla\Wt(P_c)^\top(\sin\theta_c,\cos\theta_c)^\top=R_c/(\nu(\nu-1))$ with $R_c=\sqrt{\nu^2-1}$, and $\Bbar_+(\tau_c-\epsilon)=P_c+\frac{\nu\epsilon^2}{2}(\sin\theta_c,\cos\theta_c)^\top+\frac{\nu\epsilon^3}{3}\mathbf n(\theta_c)+O(\epsilon^4)$, we get $\Wt(\Bbar_+(\tau_c-\epsilon))=\tau_c+R_c\epsilon^2/(2(\nu-1))+O(\epsilon^3)$, below the folded value $\tau_c+2\epsilon+O(\epsilon^3)$. The tributary therefore carries the far side next to the cusp. With $\ell_c=\nu\epsilon^2/2$ the arc length from $\Bbar_+(\tau_c-\epsilon)$ to $P_c$ and $\mu>0$ the distance from $P_c$ along $\mathbf n(\theta_c)$,
\begin{align*}
\Delta W&=\epsilon+\frac{R_c}{2(\nu-1)}\epsilon^2+O(\epsilon^3)=\sqrt{2\ell_c/\nu}+O(\ell_c),\\
W(P_c+\mu\,\mathbf n(\theta_c))-\tau_c&=\frac{\mu}{\nu(\nu-1)}+O(\mu^2),
\end{align*}
instead of the laws $3\sqrt{2\ell_c/\nu}$ and $(6\mu/\nu)^{1/3}$ of strict F2b (Appendix~\ref{app:cusp}); both expansions agree with direct evaluation of $W$ (\OR{S9}).
\section{Numerical comparisons}\label{sec:num}
We compare the value with two independent solvers that use no characteristic information and test the strategies in closed loop.

\paragraph{Level-set solver} The escape time is approximated through the reach-time equation $\partial_t\zeta=\min\{0,\tilde H(\x,\nabla\zeta)\}$, $\tilde H(\x,\mathbf p)=-\nu\|\mathbf p\|+|p_y|+|p_xy-p_yx|$, initialized by the signed distance to $\partial K$, positive inside $K$ \citep{mitchell2005}, with fifth-order WENO differences \citep{jiang2000}, a local Lax--Friedrichs flux and third-order TVD Runge--Kutta time stepping \citep{shu1988} at CFL number $0.45$. The escape time at a node is the first $t$ with $\zeta<0$, interpolated linearly within the time step, which mirrors the strict-exit convention; implementation details and run configurations are in the Supplementary Material, Sects.~S1--S2.

\paragraph{Regime F2b} Figure~\ref{fig:levelsets} and Table~\ref{tab:validation} compare $W$, evaluated by bracketing every root of~\eqref{eq:Gscalar}, with the level-set solution away from the rays, the axis, the barrier and the kinks. Primary errors are below $10^{-7}$ in median and $10^{-4}$ at most; tributary errors are a few $10^{-3}$, and the tributary offset approximately halves with $h$ and changes sign with the alignment of the grid with the axis. These comparisons test the smooth parts of the value preferentially (\OR{S1}, \OR{S3}).

\begin{figure}[!htbp]
\centering
\includegraphics[width=0.9\textwidth]{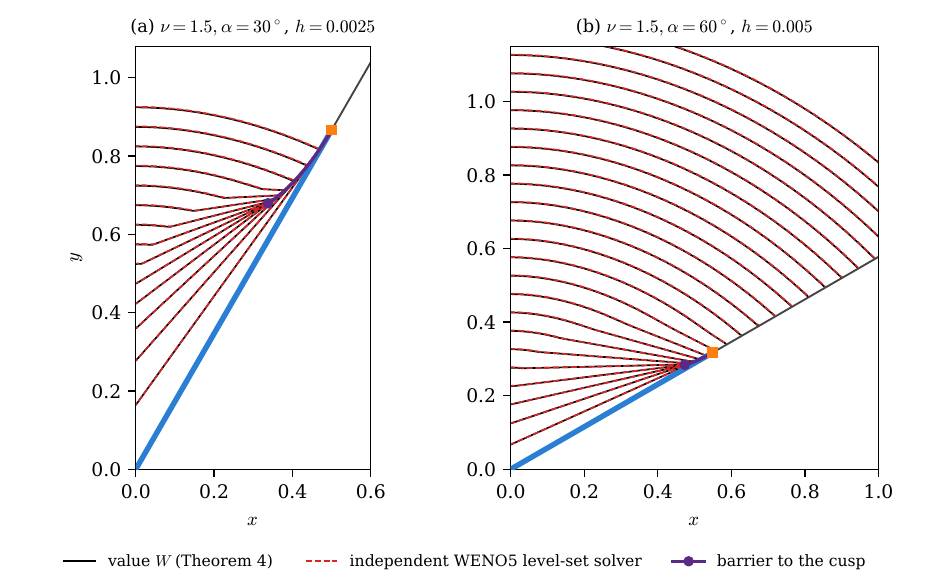}
\caption{Regime F2b: level sets of the value of Theorem~\ref{thm:F2b} (solid) and of the independent WENO5 solution (dashed) in the right half of the cone, at levels spaced by $0.1$. (a) $\nu=1.5$, $\alpha=30^\circ$, grid spacing $h=0.0025$; (b) $\nu=1.5$, $\alpha=60^\circ$, $h=0.005$. $\UP$ thick, $B_+$ marked; barrier from $B_+$ to the cusp $P_c$ (dot), across which both solutions jump.}
\label{fig:levelsets}
\end{figure}

\begin{table}[!htbp]
\centering\small
\caption{Level-set comparisons in regime F2b away from rays, the axis, barriers and kinks. Primary columns give the median relative error $|W_{\rm LS}-W|/\max(W,0.05)$, with $W_{\rm LS}$ the level-set value, and the maximum absolute error; tributary columns give the signed mean and the maximum absolute error.}
\label{tab:validation}
\begin{tabular}{lccccc}
\toprule
$(\nu,\alpha)$ & $h$ & prim.\ median & prim.\ max & trib.\ mean & trib.\ max\\
\midrule
$(1.5,30^\circ)$ & 0.0025 & $2.1\times10^{-8}$ & $3.8\times10^{-5}$ & $-0.0026$ & $0.0026$\\
$(1.5,60^\circ)$ & 0.005 & $3.2\times10^{-8}$ & $4.6\times10^{-5}$ & $-0.0020$ & $0.0020$\\
$(2,15^\circ)$ & 0.005 & $5.9\times10^{-8}$ & $6.0\times10^{-5}$ & $-0.0038$ & $0.0038$\\
\bottomrule
\end{tabular}
\end{table}

\paragraph{Regime F1: a test against a proved value} By Theorem~\ref{thm:value}, the value on the axis above $I$ is $W_T=W_U$ for $(\nu,\alpha)=(1.1,10^\circ)$. On a fixed box and at $101$ fixed axis points $0.575\le y\le0.60$, the level-set mean error at $h=0.001$ is $+0.00522$ with a grid node on the axis and $+0.21670$ with a half-cell shift (Fig.~\ref{fig:F1}c); the error grows with the distance of the nearest grid column from the axis, and its excess over the aligned result decreases with $h$ (\OR{S4}). A semi-Lagrangian discretization of the full game \citep{falcone2014}, with both orders of the players' optimizations, $48$ evader headings and $31$ pursuer controls, brackets $W_T$ at all tested resolutions with a narrowing bracket (Table~\ref{tab:SLfull}), although the bracket is not certified, and its first-order extrapolants lie $0.008$--$0.013$ above $W_T$ (\OR{S5}). Both offsets are therefore errors of the discretizations with respect to the value. The level-set offset is consistent with insufficient resolution near the universal line, where the tributary channel between $\Dpt^\pm$ is far narrower than the grid near $E$, but we have not isolated the mechanism.

\begin{table}[!htbp]
\centering
\caption{Full-game semi-Lagrangian computation for $\nu=1.1$, $\alpha=10^\circ$; P, E: pursuer or evader first. Axis: mean of $V-W_T$ on the axis column $0.575\le y\le0.60$. Tributaries: mean of $V-W_T$ at 214 fixed tributary points. Primaries: largest $|V-W|$ at 536 fixed primary points. Last row: first-order extrapolation of the mean errors.}
\label{tab:SLfull}
\setlength{\tabcolsep}{4pt}
\small
\begin{tabular}{lccc}
\toprule
$h$ & axis, P / E & trib.\ mean, P / E & prim.\ max\\
\midrule
0.002 & $-0.122$ / $+0.036$ & $-0.116$ / $+0.041$ & $0.047$\\
0.001 & $-0.050$ / $+0.026$ & $-0.046$ / $+0.030$ & $0.011$\\
0.0005 & $-0.020$ / $+0.017$ & $-0.016$ / $+0.021$ & $0.004$\\
\midrule
extrap. & $+0.009$ / $+0.008$ & $+0.013$ / $+0.012$ & \\
\bottomrule
\end{tabular}
\end{table}

\paragraph{Closed-loop tests of the strategies} We simulated the strategies of Theorems~\ref{thm:feedback} and~\ref{thm:F2a} from states drawn at random and placed near the barrier, $E$, $Z$, the channel and $\U$, against relay, constant, random and greedy opponents (\OR{S8}). The evader's overshoot stays within the switching tolerance plus the step size, and no evader beats $W$ against the relay (Table~\ref{tab:closedloop}). With the evader's control held over sampling periods $h\le0.02$ the largest overshoot is $3.8\times10^{-3}$; for $h=0.04$ some trajectories pass $E$ between sampling instants, because the band of $A_\varepsilon$ around $E$ on the axis is only about $3\times10^{-3}$ wide ($W_P-\Wt$ vanishes to second order there), as Lemma~\ref{lem:sampE} anticipates. From 66 states within $10^{-2}$ to $10^{-6}$ of the barrier in label, the committed heading of Lemma~\ref{lem:commit} exits within one integration step of $W$ once the integration resolves the near-tangential exit (\OR{S8}). \citet{braun2026} study the effect of sample-and-hold implementation numerically for the prying-pedestrian surveillance--evasion game, with feedback strategies learned from open-loop solutions; here the sampled strategies carry proved guarantees (Theorems~\ref{thm:value}, \ref{thm:F2a} and~\ref{thm:F2b}), and the simulations illustrate them.

\begin{table}[!htbp]
\centering\small
\caption{Closed-loop tests of Theorems~\ref{thm:feedback} (F1, five pairs including $\alpha=\alpha_F(1.1)-10^{-8}$) and~\ref{thm:F2a} (F2a, nine pairs, four of them close to $\alpha_Z$): largest overshoot of the evader feedback against the tested pursuers and smallest margin of the relay against the tested evaders, over all pairs. Per-pair results are in the Supplementary Material, Table~S12.}
\label{tab:closedloop}
\begin{tabular}{lccc}
\toprule
regime & initial states & $\max(t_f-W)$, evader & $\min(t_f-W)$, relay\\
\midrule
F1 & 270 & $1.9\times10^{-3}$ & $+5.8\times10^{-6}$\\
F2a & 421 & $2.0\times10^{-3}$ & $+5.6\times10^{-6}$\\
\bottomrule
\end{tabular}
\end{table}

The closed-form arcs agree with a numerical integration of the characteristic equations to $10^{-12}$ (\OR{S6}); the identities used in the proofs were checked symbolically, the covering of Lemma~\ref{lem:cover} on 660 parameter pairs across F1 (\OR{S7}), and the structure used in Theorem~\ref{thm:F2a} on 28 pairs across F2a (\OR{S9}).

\section{Comparison with the conic game for a Dubins pursuer}\label{sec:lewin}
For a slower evader the Reeds--Shepp synthesis is simpler than the forward-only one of \citet{lewin1979conic} (Section~8 of Part~I); for a faster evader the two become structurally close. In their notation $\nu=1/w$ and $\Theta=\alpha$; both games have the reduced dynamics~\eqref{eq:reduced} and differ only in the pursuer's control set, $v_p=1$ for the Dubins car. For convex cones their Regions~I and~II contain no barrier and compare directly with our synthesis (Table~\ref{tab:lewin}). Above the distinguished point of the axis the two pursuers resolve the same situation in opposite ways: the Dubins car must turn to one side, so the axis above the graze point $G$ is a pursuer's dispersal line, whereas the Reeds--Shepp car backs away with singular steering $\omega_p=0$ along the universal line $\U$, a choice of side that becomes a sliding mode of the relay. The barrier has the same circle-involute geometry in both games, an involute of the circle of radius $\nu=1/w$ about the turning centre ending at a cusp; \citet{lewin1979conic} find it only in obtuse cones, and reversal, by shortening the usable part from $r(\nu+\sin\alpha)$ to $r(\nu-\sin\alpha)$, brings it into convex cones. Finally, they exclude knowledge of the opponent's current control and treat singular arcs as limits of $\varepsilon$-optimal play; here the singular steering on $\U$ is realized by the relay, and Theorems~\ref{thm:value}, \ref{thm:F2a} and~\ref{thm:F2b} show that sampled state feedback attains the value to arbitrary precision, so the value is the same with and without access to the opponent's current control. In the game between two identical Dubins cars \citep{ruiz2024dubins} a universal surface is singular for the evader's control; $\U$ is singular for the pursuer's, and reversal makes that control admissible.

\begin{table}[!htbp]
\centering
\caption{Conic escape game for convex cones, $0<\alpha<\pi/2$: Dubins pursuer (DC; \citealp{lewin1979conic}, Regions~I--II, every $w=1/\nu$) and Reeds--Shepp pursuer with a faster evader (RS, $\nu>1$, this paper). RS entries are proved for every $\alpha$ (Theorems~\ref{thm:value}, \ref{thm:F2a} and~\ref{thm:F2b}).}
\label{tab:lewin}
\small
\renewcommand{\arraystretch}{1.15}\setlength{\tabcolsep}{5pt}
\begin{tabularx}{\linewidth}{@{}l>{\raggedright\arraybackslash\hsize=0.95\hsize}X>{\raggedright\arraybackslash\hsize=1.05\hsize}X@{}}
\toprule
 & DC & RS, $\nu>1$\\
\midrule
Usable part & $\rho<r(\nu+\sin\alpha)$ & $\rho<r(\nu-\sin\alpha)$\\
Holding region & none (no barrier for $\Theta\le\pi/2$) & none (Part~I, Prop.~3(c))\\
Point dividing the axis & graze point $G$ (a primary is tangent to the axis) & $E$ (subsolution test fails, $\lambda_y=1/(\nu-1)$)\\
Axis below it & evader's dispersal line & evader's dispersal line $\De$\\
Axis above it & pursuer's dispersal line & universal line $\U$, pursuer backs away ($\omega_p=0$)\\
Surface leaving the axis there & equivocal surface, to the ray or to a switch curve & evader's dispersal line $\Dpt$, to the barrier or to the $\NUP$\\
Switch curves & in Region~II & none\\
State constraint on the ray & yes (beyond $S$) & no\\
Barrier & only for $\Theta>\pi/2$: involute of the $1/w$-circle, ending at a cusp & involute of the $\nu$-circle; the jump ends at $Z_\pm$ (F1, F2a) or at the cusp $P_c$ (F2b)\\
Singular pursuer control & intermediate turn rate (equivocal surface, state constraint) & $\omega_p=0$ on $\U$, realized by relay sliding\\
Information pattern & no current-control knowledge; singular arcs as limits of $\varepsilon$-optimal play & value approached arbitrarily closely by sampled state feedback at every state\\
\bottomrule
\end{tabularx}
\end{table}

\section{Conclusion}\label{sec:concl}
Together with Part~I, this solves the conic escape game for every $\nu\le1$ and every $\nu>1$ at every $\alpha\in(0,\pi/2)$; for $\nu\le1$ only the evader's strategy value at one corner was exceptional, and for $\nu>1$ there is no exception. For the faster evader the escape time is organized by a universal line on the front axis, on which the reversing pursuer backs away, by its tributaries, and by a $\BUP$ characteristic that ends at a cusp; reversal is what brings this cusped barrier, which \citet{lewin1979conic} find for a Dubins pursuer only in obtuse cones, into convex cones.

Regime F2 is not one regime: the geometric threshold $\nu\sin\theta_c=1$ says whether the $\BUP$ characteristic reaches the axis before its cusp, but whether the value jumps up to the cusp is decided by the sign of $\Wt(P_c)-\tau_c$. In F2a the F1 synthesis and strategies remain optimal. In F2b, where folded primaries carry the value near the cusp, the feedback argument of F1 breaks down near the continuation of the barrier, and the evader commits in advance to the heading of a primary, certified against every pursuer control by the scalar equation of the primaries; proving the relay bound for an auxiliary envelope and recovering the covering afterwards avoids any classification of the dispersal lines, whose topology in F2b we describe only for the examples. On the numerical side, the proved value shows that the persistent offsets of level-set and semi-Lagrangian solvers on the universal line are discretization errors, which suggests the universal line as a benchmark for solvers of games with singular surfaces. Extensions to half-planes, nonconvex cones and finite sensor range are separate problems.

\section*{Supplementary information}
The Supplementary Material, available as an ancillary file of this arXiv submission, contains the level-set and semi-Lagrangian methods and run configurations, the F2 grid comparison, the F1 level-set and semi-Lagrangian studies against the value of Theorem~\ref{thm:value}, characteristic quantities and residuals, the symbolic and numerical checks of the F1 proofs, the closed-loop simulations of the strategies, and the jump table and an additional portrait for regime F2.

\section*{Statements and Declarations}
\paragraph{Competing interests} The author declares no competing interests.
\paragraph{Author contributions} U.~Ruiz is the sole author of this work.
\paragraph{Funding} This research did not receive any specific grant from funding agencies in the public, commercial, or not-for-profit sectors.
\paragraph{Data availability} The solvers, raw numerical grids, run settings and scripts that reproduce every computation, table and figure, including the closed-loop simulations, are available from the author upon reasonable request.
\paragraph{Use of generative AI} During the preparation of this work the author used Claude (Anthropic) in order to edit the text, check derivations and numerical consistency, draft parts of the proofs, and write and run code for the numerical computations and figures. After using this tool, the author reviewed and edited the content as needed and takes full responsibility for the content of the published article.

\appendix
\section{Proof of Theorem~\ref{thm:value}}\label{app:value}
We use the pairing of sample-and-hold controls with nonanticipative strategies of Appendix~B of Part~I, which does not depend on $\nu$, and denote by $M$ a bound for $\|\dot\x\|$ on the bounded sets considered. By Theorem~\ref{thm:feedback}, it remains to show $V_E\ge W$ and $V_P\le W$.

\begin{lemma}[Sampled relay]\label{lem:sampP}
Let $\x_0\in K$ and $T<W(\x_0)$, and let the pursuer play $v_p=-1$, $\omega_p=-\sgn x(kh)$ on $[kh,(k+1)h)$, with $\sgn0:=1$. There is $h_0>0$ such that for $h<h_0$ and every measurable evader control, $\x(t)\in K$ on $[0,T]$. Hence $V_E\ge W$ on $K$.
\end{lemma}
\begin{proof}
Argue by contradiction as in Lemma~5 of Part~I: there are $h_n\downarrow0$, evader controls $\mathbf u_n$ and solutions $\x_n$ leaving $K$ at times $t_n\le T$. Along a subsequence, $\x_n\to\bar\x$ uniformly, $t_n\to t^*\le T$, $\bar\x$ is a relay solution (the identification of the limit uses only the absence of products of the two players' controls in~\eqref{eq:reduced}), and $\bar\x(t^*)\in\partial K$. Contact with the $\NUP$ is excluded as in Lemma~5 of Part~I. Contact with $\UP\cup\{\mathbf 0\}$ approached from $\Pset$ gives $t^*\ge\lim_{t\uparrow t^*}\varphi(t)\ge W(\x_0)>T$ by Lemma~\ref{lem:mono}. The remaining case is $\bar\x(t^*)=B_+$ (or $B_-$), with $\bar\x(t)\in N$ for $t\in[t^*-\kappa,t^*]$. By Lemma~\ref{lem:inv}, $\bar\x$ cannot pass from $\{D>0\}$ to $\{D\le0\}$ or from $\{D=0\}$ to $\{D<0\}$, so one of the following holds.
\begin{itemize}
\item $D(\bar\x(t))<0$ near $t^*$. Then $W(\bar\x(t))\to0$ and the previous argument applies.
\item $D(\bar\x(\bar t))>0$ for some $\bar t\in[t^*-\kappa,t^*)$. Then $D(\x_n(\bar t))>0$ for large $n$. Near $\Lambda_+$ every sample of $\x_n$ lies in $\{x>x_Z/2\}$ once $Mh_n<x_Z/2$, so $\x_n$ plays exactly $(-1,-1)$ there, and Lemma~\ref{lem:inv} keeps $\x_n$ in $\{D>0\}\cap N\subset K$, contradicting $t_n\to t^*$.
\item $D(\bar\x(t))=0$ on $[t^*-\kappa,t^*]$. Then $\frac{d}{dt}D=0$ a.e.\ forces $\mathbf u=-\nu\mathbf n$, the unique minimizer, so $\bar\x$ moves along $\Bbar_+$ with its retro-time decreasing at unit rate \PI{Lemma~3(c)}. Let $t_a$ be the start of the maximal interval ending at $t^*$ on which $\bar\x\in\Lambda_+$. If $t_a>0$, then by Lemma~\ref{lem:inv} the state cannot arrive from $\{D>0\}$, so either it arrives from $\{D<0\}$, the lens, at a point of $\Lambda_+\setminus\{Z_+\}$, or $\bar\x(t_a)=Z_+$, reached from the lens or along the continuation of the $\BUP$ characteristic beyond $Z_+$. In the first case $W(\bar\x(t))$ tends to the lens trace $\tau(\bar\x(t_a))$ as $t\uparrow t_a$; in the second, $W$ is continuous at $Z_+$ with $W(Z_+)=\tau_Z$ (Lemma~\ref{lem:P}(a)). In both cases Lemma~\ref{lem:mono} on $[0,t_a)$ gives $\tau(\bar\x(t_a))+t_a\ge W(\x_0)$, and following $\Bbar_+$ to $B_+$ takes the time $\tau(\bar\x(t_a))$, so $t^*=t_a+\tau(\bar\x(t_a))\ge W(\x_0)>T$. If $t_a=0$, then $\x_0\in\Lambda_+$, and Lemma~\ref{lem:inv} applied to $\x_n$ itself, with $D(\x_0)=0$ and the constant control near $\Lambda_+$, gives $D(\x_n)\ge0$ up to $t^*+\kappa'$, so $\x_n$ stays in $K$ near $B_+$, a contradiction.
\end{itemize}
Against every evader strategy, the sampled relay yields a pursuer control for which the evader does not exit before $T$; hence $V_E\ge T$ for every $T<W(\x_0)$.
\end{proof}

Unlike \PI{Lemma~5}, there is no corner case: the barrier arcs stay in $\{|x|\ge x_Z\}$, where the sampled relay does not switch.

\begin{lemma}[Uniform sampled primary feedback]\label{lem:unif}
Assume~\eqref{eq:F1}, and let $0<\bar\rho<\rho_B$ and $\varepsilon>0$. Let $\overline\Cset(\bar\rho)$ be the image under~\eqref{eq:primary} of the compact strip $\{0\le\rho_f\le\bar\rho,\ 0\le\tau\le\tau_D(\rho_f)\}$. There is $h_1>0$, depending only on $\bar\rho$ and $\varepsilon$, such that for every $\x_1\in\overline\Cset(\bar\rho)$, every $h<h_1$ and every measurable pursuer control, the evader that holds on $[kh,(k+1)h)$ the right-branch primary feedback evaluated at $\x(kh)$ leaves $K$ by the time $W_P(\x_1)+\varepsilon$.
\end{lemma}
\begin{proof}
We follow Lemma~6 of Part~I and keep track of the constants. Let $\hat\rho=(\bar\rho+\rho_B)/2$. On $\overline\Cset(\hat\rho)$, $G_\tau=\rho_B-\rho_f\ge\rho_B-\hat\rho>0$, so by the implicit function theorem $W_P$, $\rho_f$ and the feedback $\mathbf k(\x)=\nu(\cos\theta,-\sin\theta)^\top$, $\theta=\alpha+W_P(\x)$, extend smoothly to an open neighbourhood of this compact set, across the axis and the $\UP$. Fix a compact neighbourhood $Q$ of $\overline\Cset(\hat\rho)$ inside it, let $L$ be a Lipschitz constant of $\mathbf k$ on $Q$ and $M$ a bound of $\|\dot\x\|$ on $Q$ over all controls, and let $T_0=\tau_I+1$; note that $W_P<\tau_I$ on $\overline\Cset(\hat\rho)$. All these quantities depend only on $\bar\rho$.

\emph{Margins.} On the axis part of $\overline\Cset(\hat\rho)$, $\dot x\ge\nu C-y=[g+S(1-\nu S)]/C\ge m_1>0$ against every pursuer control, because $g\ge\rho_B-\hat\rho$ and $\nu S\le\nu\sin(\alpha+\tau_D(\hat\rho))<1$. On the part of $\UP_+$ with $\rho\le\hat\rho$ the outward angular velocity is at least $\nu-\rho-\sin\alpha\ge\rho_B-\hat\rho$, by Eq.~(8) of Part~I, and at the apex, where the bearing is undefined, the feedback $\nu(\cos\alpha,-\sin\alpha)^\top$ gives an outward normal speed $\nu+v_p\sin\alpha\ge\nu-\sin\alpha$ through $M_+$. By continuity these margins persist, halved, on a fixed strip around each of these edges and for held errors up to some $e_1>0$.

\emph{Estimates.} While $\x(t)\in Q\cap\{x\ge0\}$, the held control differs from $\mathbf k(\x(t))$ by at most $LMh$. Since Eqs.~(27)--(28) of Part~I hold for the exact feedback and every pursuer control (Lemma~\ref{lem:prifb}), $\dot W_P\le-1+C_0h$ and $\dot\rho_f\le C_0h$ there, with $C_0$ depending only on $Q$.

\emph{Choice of $h_1$.} Take $h_1$ with $LMh_1<e_1$, $C_0h_1T_0<\hat\rho-\bar\rho$, $C_0h_1<\frac12$ and $C_0h_1\tau_I/(1-C_0h_1)\le\varepsilon$. Let $h<h_1$ and $\x_1\in\overline\Cset(\bar\rho)$. Up to time $T_0$, $\rho_f$ increases by less than $\hat\rho-\bar\rho$, so the state does not reach the edge $\rho_f=\hat\rho$; by the axis margin it enters $\{x>0\}$ immediately if it starts on the axis and does not return. A first-exit argument therefore keeps the state in $\overline\Cset(\hat\rho)$ until it reaches $\UP$ or the apex, which happens by the time $W_P(\x_1)/(1-C_0h)\le W_P(\x_1)+\varepsilon<T_0$, because $W_P\ge0$ on $\overline\Cset(\hat\rho)$ with equality only there. At the contact the outward margin gives a strict exit immediately.
\end{proof}

\begin{lemma}[Sampled evader]\label{lem:sampE}
Let $\x_0\in K$ and $\varepsilon>0$. If $\x_0\in\Pset$, let the evader hold the primary feedback of the half containing $\x_0$ (the right one on the axis), as in Lemma~\ref{lem:unif}. If $\x_0\in\Tset$, let it hold $\mathbf u_T(\x(kh))$ on $[kh,(k+1)h)$ until the first sampling instant $t_1=kh$ with $\x(kh)\in A_\varepsilon$, and from then on the sampled primary feedback of the half containing $\x(t_1)$. There is $h_0>0$ such that for $h<h_0$ and every measurable pursuer control, $t_f\le W(\x_0)+3\varepsilon$. Hence $V_P\le W$ on $K$.
\end{lemma}
\begin{proof}
For $\x_0\in\Pset$ this is Lemma~\ref{lem:unif} with $\bar\rho=\rho_f(\x_0)$, by symmetry if $x_0<0$. Let $\x_0\in\Tset$.

\emph{Localization.} We fix the compact set before any constant that depends on it. Let $T_0=2(\Wt(\x_0)-W_E)+1$ and $\rho_0=\|\x_0\|$. For $\rho=\|\x\|>0$ the rotational terms of~\eqref{eq:reduced} cancel and $\dot\rho=(xu_x+yu_y-v_py)/\rho\le\nu+1$, so $\|\x(t)\|\le\rho_1=\rho_0+(\nu+1)T_0$ on $[0,T_0]$ for every control. On $\Tset$, $s\ge y_E$ gives $R\ge y_E$, which with $|x|\le y\tan\alpha$ bounds $y$ below by some $y_1>0$; the closure of $\Pset\cap\{\rho_f\ge\rho_E\}$ and the $\delta_0$-band of $\Gamma$ in $\Omega_0$ are compact subsets of $\{y>0\}$ for some $\delta_0>0$. Let $Q=\{\x\in K:\|\x\|\le\rho_1\}\cap\bigl(\overline\Tset\cup\overline{\Pset\cap\{\rho_f\ge\rho_E\}}\cup\text{band}\bigr)$, a compact subset of $\{y>0\}$, since $s\ge y_E$ also on $\overline\Tset$ (Lemma~\ref{lem:cover}). Only now let $M$ bound $\|\dot\x\|$ on $Q$, $L$ be a Lipschitz constant of $\mathbf u_T$ on $Q$, and take $\delta\le\delta_0$ as in Lemma~\ref{lem:sep}. Once $Mh<\delta$, the state does not reach $\Omega_0$ at distance $\ge\delta$ from $\Gamma$ before the switch: by Lemma~\ref{lem:sep}(c) this would require a time interval of length $\ge\delta/M>h$ in $A_\varepsilon$, which contains a sampling instant. Hence, up to the switch and up to $T_0$, the state stays in $Q$.

\emph{Decrease.} The held control differs from $\mathbf u_T(\x(t))$ by at most $LMh$, so by~\eqref{eq:WTdot} and $\|\nabla\Wt\|=1/(\nu-1)$, $\frac{d}{dt}\Wt\le-1+C_0h$ with $C_0=LM/(\nu-1)$, uniformly over pursuer controls.

\emph{Switch.} Let $\hat t=(\Wt(\x_0)-W_E)/(1-C_0h)$, which is $<T_0-h$ once $C_0h\le\frac12$ and $h<1$, and let $kh$ be the first sampling instant $\ge\hat t$ before the exit. Then $\Wt(\x(kh))\le W_E$, so $\x(kh)\notin\Tset\cup\Sigma^\pm_\eta$ by Lemma~\ref{lem:cover}, and since $E\in A_\varepsilon$, \eqref{eq:Pdecomp} gives $\x(kh)\in A_\varepsilon\cup(\Omega_0\setminus A_\varepsilon)$, where the second alternative is excluded by the localization. Hence $t_1\le\hat t+h$ and
\[
W_P(\x(t_1))\le\Wt(\x(t_1))+\varepsilon\le\Wt(\x_0)-(1-C_0h)t_1+\varepsilon .
\]

\emph{After the switch.} $\x(t_1)\in A_\varepsilon\subset\overline\Cset(\rho_B-\eta)$ or its mirror image, so Lemma~\ref{lem:unif} with $\bar\rho=\rho_B-\eta$ gives a threshold $h_1$, independent of $\x(t_1)$, such that $t_f-t_1\le W_P(\x(t_1))+\varepsilon$ for $h<h_1$. Altogether $t_f\le\Wt(\x_0)+2\varepsilon+C_0ht_1\le W(\x_0)+3\varepsilon$ for small $h$.

Against every pursuer strategy, the sampled evader yields a control with $J\le W(\x_0)+3\varepsilon$; hence $V_P\le W$.
\end{proof}

\begin{proof}[Proof of Theorem~\ref{thm:value}]
Theorem~\ref{thm:feedback} gives $V_E\le W\le V_P$, and Lemmas~\ref{lem:sampP}--\ref{lem:sampE} give $V_E\ge W$ and $V_P\le W$. Both lemmas use sampled state feedback only, which gives the last statement \PI{Remark~2}.
\end{proof}

\begin{lemma}[Sampled strategies in F2b]\label{lem:sampF2b}
In regime F2b and for $\alpha=\alpha_Z(\nu)$, $V_E\ge W$ and $V_P\le W$ on $K$.
\end{lemma}
\begin{proof}
\emph{Pursuer.} The proof of Lemma~\ref{lem:sampP} applies with $\Lambda_+$ replaced by the closed barrier arc $\Bbar_+[0,\tau_c]$, $Z_+$ by $P_c$, and Lemma~\ref{lem:mono} by Proposition~\ref{prop:relayF2b}. The barrier arc stays in $\{x\ge x(P_c)\}$ with $x(P_c)>0$, so the sampled relay plays exactly $(-1,-1)$ near it once $Mh_n<x(P_c)/2$. Contact of the limit relay solution with $B_+$ is treated by Lemma~\ref{lem:inv} near $B_+$, where $q>0$; a maximal barrier interval beginning at a positive time starts either from the lens, with the trace as left limit, or at $P_c$, where $V$ is continuous with $V(P_c)=\tau_c$, and in both cases Proposition~\ref{prop:relayF2b} gives the contradiction.

\emph{Evader.} A committed heading (Lemma~\ref{lem:commit}) is an open-loop program in the reduced frame, fixed at a sampling instant, so it needs no further sampling; the same holds at the cusps. Let $\x_0$ be a state outside $\Omega_0$ where $W=\Wt$, and $\varepsilon>0$. Choose $\eta>0$ with $\rho_E+\eta<\rho_a$ and extend the core chart to the labels $\rho_f\le\rho_E+\eta$ up to their first axis contacts; on its closure $g\ge\rho_B-\rho_E-\eta>0$ and $\tau\le\tau_D(\rho_E+\eta)<\tau_a$, so it is regular, and so is its mirror image. Let
\[
N_\varepsilon=\{\x\ \text{in the closed extended chart or its mirror image}:\ P_0(\x)\le\Wt(\x)+\varepsilon/3\},
\]
where $P_0$ is the time of the (extended) core primary. Since $P_0\le\Wt$ on $\Gamma_E$ and both are uniformly continuous there, $N_\varepsilon$ contains a band $\{\x\in\Omega_0:\operatorname{dist}(\x,\Gamma_E)<\delta\}$ for some $\delta>0$, including a relative neighbourhood of $E$. Fix the horizon $T_0=2\Wt(\x_0)+2$; since $\frac{d}{dt}\|\x\|\le\nu+1$, the state stays in a ball before $T_0$, and outside $\Omega_0$, or in the band, it stays at a positive distance from the apex. This gives a compact set $Q\subset\{y>0\}$, fixed before the constants: let $M=\sup_Q\|\dot\x\|$ over all controls and $L$ a Lipschitz constant of $\mathbf u_T$ on $Q$. The evader holds $\mathbf u_T(\x(kh))$ on $[kh,(k+1)h)$ until the first sampling instant $t_1$ with $\x(t_1)\in N_\varepsilon$, and then commits to the heading of the extended core primary through $\x(t_1)$. Before $t_1$, $\|\mathbf u_T(\x(kh))-\mathbf u_T(\x(t))\|\le LMh$, so $\frac{d}{dt}\Wt\le-1+C_0h$ with $C_0=LM/(\nu-1)$. Take $h<\delta/M$, $C_0h<\frac12$ and $C_0hT_0<\varepsilon/3$. The state cannot remain outside $\Omega_0$ beyond $2\Wt(\x_0)$, since $\Wt\ge m>0$ there; to enter $\Omega_0$ it crosses the band, which takes a time at least $\delta/M>h$, so a sampling instant falls in $N_\varepsilon$; hence $t_1\le T_0$. Then Lemma~\ref{lem:commit} gives $t_f\le t_1+P_0(\x(t_1))\le t_1+\Wt(\x(t_1))+\varepsilon/3\le\Wt(\x_0)+C_0ht_1+\varepsilon/3<\Wt(\x_0)+\varepsilon$.

Both constructions use sampled state feedback or open-loop controls only, so they pair with nonanticipative strategies as in Appendix~B of Part~I.
\end{proof}

The threshold $h_0$ of Lemma~\ref{lem:sampE} is governed by the width of the band of $A_\varepsilon$ around $E$. Along the axis, $W_P-\Wt$ vanishes to second order at $E$, because $\lambda_y=1/(\nu-1)$ there. Differentiating the axis equation gives $W_P(0,y)-\Wt(0,y)=c_E(y-y_E)^2+O(|y-y_E|^3)$ with $c_E=C_*[(2\nu-1)+(\nu-1)S_*]/[2S_*(1+S_*)(\nu-1)^3]$, which is $528$ for $(1.1,10^\circ)$; the cubic term is not negligible at the scale of the band, and solving $W_P-\Wt=\varepsilon=10^{-3}$ directly gives a band $0.55008\le y\le0.55284$ of width $2.75\times10^{-3}$, and a trajectory running down $\U$ at speed $\nu-1$ is sampled in it only if $h\lesssim0.028$. This heuristic scale, which ignores the other constants of Lemma~\ref{lem:sampE}, matches the closed-loop tests of Section~\ref{sec:num}.

\section{Expansions at the cusp}\label{app:cusp}
We write $\Bbar_+(\tau)$ for the point of the $\BUP$ characteristic at retro-time $\tau$, $\theta_0=\alpha+\tau$, and $q(\tau)=\cos\alpha-\nu\tau$, so that $q(\tau_c)=0$. The expansions below concern the folded-primary branch $W_{\rm fold}$ near $P_c$. They describe the jump of $W$ only where that branch carries the far-side value, which near $P_c$ is the case in strict F2b, $\alpha>\alpha_Z(\nu)$; in F2a the tributary carries the value near $P_c$ and $W$ is smooth there, and on $\alpha=\alpha_Z(\nu)$ the tributary carries the far side, with the laws given in Section~\ref{sec:F2b}. We consider two approaches to the cusp: along $\Bbar_+$ toward $P_c$, and from $P_c$ along the normal $\mathbf n(\theta_c)$ to the far side.

\begin{lemma}[Scalar equation along the barrier]\label{lem:barG}
For $\tau\ge0$ and $\delta\in\mathbb R$,
\begin{equation}\label{eq:barG}
G(\tau+\delta;\Bbar_+(\tau))=\nu(\delta-\sin\delta)-q(\tau)(1-\cos\delta),
\end{equation}
and the launch parameter~\eqref{eq:rhof} of the primary through $\Bbar_+(\tau)$ at retro-time $\tau+\delta$ satisfies $\rho_B-\rho_f=\nu(1-\cos\delta)-q(\tau)\sin\delta$. Moreover, for every retro-time $\tau'$ and $\theta'=\alpha+\tau'$, $G(\tau';\x+\mathbf w)=G(\tau';\x)+w_x\cos\theta'-w_y\sin\theta'$, and $\rho_f$ changes by $w_x\sin\theta'+w_y\cos\theta'$.
\end{lemma}
\begin{proof}
By~\eqref{eq:primary}, $\Bbar_+(\tau)$ has $\psi=\nu-\sin\theta_0$ and $\eta=\cos\alpha-\cos\theta_0-\nu\tau$. With $\theta=\theta_0+\delta$, $x\cos\theta-y\sin\theta=-\psi\sin\delta+\eta\cos\delta$ and $x\sin\theta+y\cos\theta=\psi\cos\delta+\eta\sin\delta$. Substituting into~\eqref{eq:Gscalar} and~\eqref{eq:rhof}, the terms in $\sin\theta_0$ and $\cos\theta_0$ cancel. The last statement holds because $G$ and $\rho_f$ are affine in $\x$.
\end{proof}

\paragraph{Along the barrier} At $\tau=\tau_c-\epsilon$, $\epsilon>0$, we have $q=\nu\epsilon$, so the roots $\delta\ne0$ near $0$ of~\eqref{eq:barG} satisfy $\delta-\sin\delta=\epsilon(1-\cos\delta)$. The leading balance $\delta^3/6=\epsilon\delta^2/2$ gives $\delta\simeq3\epsilon$, and the next order gives $\delta=3\epsilon-\tfrac{9}{10}\epsilon^3+O(\epsilon^5)$. For this root $\rho_B-\rho_f=\tfrac32\nu\epsilon^2+O(\epsilon^4)>0$, so it belongs to a primary launched just below $\rho_B$ that passes through $\Bbar_+(\tau_c-\epsilon)$ with $G_\tau>0$: the folded primary. Its retro-time there is $\tau_c-\epsilon+\delta$, while the lens-side limit is $\tau_c-\epsilon$, so the jump carried by this branch is $\Delta W=W_{\rm fold}-\tau=\delta=3\epsilon+O(\epsilon^3)$. By~\eqref{eq:bartangent} the speed along $\Bbar_+$ is $|q|=\nu(\tau_c-\tau)$, so $\Bbar_+(\tau_c-\epsilon)$ lies at arc length exactly $\ell_c=\nu\epsilon^2/2$ from $P_c$, and $\Delta W=3\sqrt{2\ell_c/\nu}+O(\ell_c^{3/2})$; the Euclidean distance to $P_c$ has the same leading behaviour.

\paragraph{Across the barrier} At $P_c+\mu\,\mathbf n(\theta_c)$, $\mu>0$, Lemma~\ref{lem:barG} with $\tau=\tau_c$, $q=0$ and $\mathbf w=\mu\,\mathbf n(\theta_c)$ gives $G(\tau_c+\delta)=\nu(\delta-\sin\delta)-\mu\cos\delta$ and $\rho_B-\rho_f=\nu(1-\cos\delta)+\mu\sin\delta$. The root near $0$ is positive and satisfies $\nu(\delta-\sin\delta)=\mu\cos\delta$, so $W_{\rm fold}-\tau_c=\delta=(6\mu/\nu)^{1/3}+O(\mu)$, again with $\rho_f<\rho_B$. For $\mu=0$ the only root is $\delta=0$, with $\rho_f=\rho_B$: the folded branch merges with the $\BUP$ characteristic at $P_c$.

\end{document}